\documentclass[twoside,11pt]{article}

\usepackage[top=1in, bottom=1.25in, left=1.25in, right=1.25in]{geometry}
\usepackage{setspace} 
\usepackage{amsfonts}
\usepackage{amsmath}
\usepackage{amsthm} 
\usepackage{amssymb} 
\usepackage{epsfig}
\usepackage{url}
\usepackage{bm}
\usepackage{bbm}
\usepackage{booktabs}
\usepackage{xcolor}
\usepackage{adjustbox}
\usepackage{mathtools} 
\usepackage{siunitx}
\usepackage{tikz-cd}
\tikzcdset{ampersand replacement=\&}
\usepackage{graphicx}
\graphicspath{ {./pix} }
\usepackage{subcaption}
\usepackage[colorlinks,pagebackref=true, allcolors=blue]{hyperref} 
\usepackage{natbib}
\usepackage{float} 
\usepackage{algorithm} 
\usepackage{todonotes}

\DeclareMathOperator{\spn}{span}
\DeclareMathOperator{\Ima}{Im}
\DeclareMathOperator{\diag}{diag}

\DeclareMathOperator{\trace}{trace}
\DeclareMathOperator{\rank}{rank}

\newcommand{\e}{{\rm e}}
\newcommand{\E}{{\mathbb E}}
\newcommand{\Pa}{{\mathbb P}}
\newcommand{\Q}{{\mathbb Q}}

\newcommand{\R}{{\mathbb R}}
\newcommand{\M}{{\mathbb M}}

\newcommand{\Ccal}{{\mathcal C}}

\newcommand{\Ecal}{{\mathcal E}}

\newcommand{\Hcal}{{\mathcal H}}

\newcommand{\Mcal}{{\mathcal M}}
\newcommand{\Ncal}{{\mathcal N}}
\newcommand{\Ocal}{{\mathcal O}}

\newcommand{\Scal}{{\mathcal S}}

\newcommand{\Ucal}{{\mathcal U}}

\newcommand{\Xcal}{{\mathcal X}}
\newcommand{\Ycal}{{\mathcal Y}}
\newcommand{\Zcal}{{\mathcal Z}}

\newtheorem{proposition}{Proposition}[section]
\newtheorem{lemma}[proposition]{Lemma}
\newtheorem{theorem}[proposition]{Theorem}

\newtheorem{remark}[proposition]{Remark}

\newcommand{\trueg}{g_{\star}}

\begin{document}

\title{Kernel Density Machines\thanks{We thank Giovanni Ballarin, Lyudmila Grigoryeva, Markus Pelger, Lorenzo Rosasco, Fabio Trojani, Johanna Ziegel, and participants of the 17th International Conference on Computational and Financial Econometrics (CFE 2023), SoFiE 2024 in Rio de Janeiro, workshop on  Numerical Methods for Finance and Insurance 2025 in Milano, FinEML Rotterdam 2025, Vienna Congress on Mathematical Finance 2025, Quantitative Methods and Learning Research Seminar at University of St.Gallen, Seminar on Insurance Mathematics and Stochastic Finance at ETH Zurich, for helpful comments. Paul Schneider gratefully acknowledges the Swiss National Science Foundation grant \emph{Large-scale kernel methods in financial economics}, grant number 215528.} }
\author{Andrea Della Vecchia\footnote{EPFL, Swiss Finance Institute. Email: andrea.dellavecchia@epfl.ch} \and Damir Filipovi\'c\footnote{EPFL, Swiss Finance Institute. Email: damir.filipovic@epfl.ch} \and Paul Schneider\footnote{Universit\`a della Svizzera italiana and Swiss Finance Institute. Email: paul.schneider@usi.ch}  }

%

\maketitle

\begin{abstract}

We introduce kernel density machines (KDM), an agnostic kernel-based framework for learning the Radon–Nikodym derivative (density) between probability measures under minimal assumptions. KDM applies to general measurable spaces and avoids the structural requirements common in classical nonparametric density estimators. We construct a sample estimator and prove its consistency and a functional central limit theorem. To enable scalability, we develop Nystr\"om-type low-rank approximations and derive optimal error rates, filling a gap in the literature where such guarantees for density learning have been missing. We demonstrate the versatility of KDM through applications to kernel-based two-sample testing and conditional distribution estimation, the latter enjoying dimension-free guarantees beyond those of locally smoothed methods. Experiments on simulated and real data show that KDM is accurate, scalable, and competitive across a range of tasks.

\end{abstract}

\vspace{2ex}




\section{Introduction}\label{sec:intro}

Estimating the Radon--Nikodym derivative 
\( g_{\star} \coloneqq \frac{d\mathbb{Q}}{d\mathbb{P}} \)
of a probability measure $\mathbb{Q}$ with respect to $\mathbb{P}$
is central to many tasks in machine learning, statistics, and the applied sciences. 
Examples include conditional distribution estimation \citep{rac_08, cmereview17}, 
two-sample hypothesis testing \citep{gre_etal_12}, 
domain adaptation and transfer learning \citep{transfer10}, covariate shift \citep{shimodaira2000improving}, 
anomaly detection \citep{nearestneighbor67, hido2011statistical}, 
causal inference \citep{Imbens_Rubin_2015}, 
and generative modeling \citep{mohamed2016learning}.
We refer to $g_\star$ as the density of $\Q$ relative to $\Pa$, 
following standard usage in probability theory, see~\citet[Chapter~III.3]{jac_shi_87}.

Most existing approaches assume that $\Pa$ and $\Q$ admit Lebesgue densities, 
in which case $g_\star$ is typically called the density ratio.
This restricts applicability to purely continuous distributions and excludes many settings of practical relevance.
In contrast, we introduce a kernel-based method that applies to general probability measures on countably generated measurable spaces, 
without requiring Lebesgue densities: \emph{kernel density machines} (KDM).
This parsimony of assumptions substantially broadens the scope of density learning.

We model $g_\star$ as the sum of an exogenous prior and a function in a reproducing kernel Hilbert space (RKHS).
While RKHS methods provide strong theoretical guarantees, they are computationally demanding in large-sample or high-dimensional settings.
To ensure scalability, we incorporate Nystr\"om-type low-rank approximations that preserve essential information while reducing computational cost.
This balance of flexibility and efficiency enables KDM to handle a wide range of real-world problems, 
from structured financial modeling to complex conditional distributions in scientific applications.

We provide rigorous theoretical guarantees for KDM, including consistency of the sample estimator, a functional central limit theorem, 
and optimal error rates for the low-rank approximation. While such results are well established for classical regression and classification problems \citep{rud_cam_ros_15,della2024nystrom}, we show for the first time that, even in the context of density learning, Nyström-compressed estimators achieve the same optimal statistical guarantees as the full model, while delivering substantial gains in computational efficiency and memory usage.
As applications, we recover and extend the kernel two-sample test of \citet{gre_etal_12} 
and develop a novel, consistent estimator of conditional distributions.
We further illustrate the practical performance of KDM through experiments on simulated and real data, 
including independence tests, conditional distribution estimation, and joint distribution estimation on a discrete state space.\footnote{
A reference implementation is available at 
\href{https://drive.google.com/file/d/1r3CteWtL3g3Kj7f6YO2HnbSP4mt4DoK3/view?usp=sharing}{KDM.R}.}
The empirical results suggest that KDM is particularly effective in larger-sample and higher-dimensional settings, 
where local nonparametric kernel estimators \citep{liracine06} often struggle.

\paragraph{Related work.}
Our approach builds on recent advances in kernel-based density estimation but extends beyond existing methods by applying to general probability measures.
\citet{sugiyama2012density, sugiyama2008least} and \citet{filipovic2024adaptive} 
either restrict attention to ratios of Lebesgue densities or do not provide the theoretical guarantees and low-rank approximation error rates developed here.
\citet{JMLR:v25:23-0567, JMLR:v24:23-1004} 
derive optimal convergence rates under strong source conditions on compact subsets of $\R^d$,
while our agnostic framework accommodates general probability measures without requiring compact support or a dominating reference measure.
A separate literature on conditional mean embeddings \citep{song2008hilbert, klebanov2021rigorous, NEURIPS2020_f340f1b1} 
focuses on conditional expectations and does not provide a general mechanism for learning densities.

To the best of our knowledge, the kernel-based density estimation literature does not provide 
Nystr\"om-type low-rank approximations with statistical guarantees, 
which are essential for large-scale applications.
We fill this gap by extending the seminal results of \citet{rud_cam_ros_15} from regression to density learning 
and proving optimal error rates in this setting.
This appears to be the first instance in which optimal error rates in the sense of \citet{rud_cam_ros_15} 
are established for kernel-based density estimation, forming the main theoretical contribution of this paper.

\paragraph{Contributions.}
The main contributions of this paper are as follows.
(i) We develop KDM, an agnostic framework for kernel-based density estimation that applies to general probability measures under minimal assumptions.
(ii) We propose a sample estimator and establish its consistency together with a functional central limit theorem.
(iii) We derive optimal error rates for density estimation under Nystr\"om-type low-rank approximations, 
explicitly addressing a gap in the literature where such guarantees have not been available.
(iv) We illustrate the applicability of KDM by constructing kernel-based two-sample tests, 
including the kernel two-sample test of \citet{gre_etal_12} as a special case.
(v) We show that KDM yields a kernel-based estimator of conditional distributions, 
overcoming classical limitations caused by the curse of dimensionality in locally smoothed estimators such as \citet{liracine06}.  
Our guarantees are dimension-free and apply to general state spaces.

\paragraph{Organization.}
The remainder of the paper is organized as follows. 
Section~\ref{sec_problem} introduces the general framework of KDM for density learning. 
Section~\ref{sec:sampleestimator} presents the sample estimator and establishes its consistency together with a functional central limit theorem. It also provides a representer theorem. Section~\ref{sec:lowrank} develops Nystr\"om approximations and derives optimal error rates. 
Section~\ref{sec:applications} describes applications of KDM to kernel-based two-sample testing and conditional distribution estimation. 
Section~\ref{sec:experiments} provides empirical evaluations on simulated and real data. 
Section~\ref{sec:conclusion} concludes. 
The appendix contains additional results, technical lemmas, and all proofs.

\section{Kernel density machines: General framework}\label{sec_problem}

We first introduce the formal setup of kernel density machines (KDM), used as standing assumptions used throughout the paper. Let $\Pa$ and $\Q$ be probability measures on some countably generated\footnote{Assuming a countably generated $\sigma$-algebra implies that the $L^2$-spaces studied below are separable. See, e.g., \citet[Theorem 4.13]{bre_11}.} measurable space $\Zcal$, where $\Q$ is absolutely continuous with respect to $\Pa$,
\begin{equation}\label{assu1}
    \Q \ll \Pa.
\end{equation}
We denote the density by $g_{\star}\coloneqq \frac{d\Q}{d\Pa}$.

The goal of this paper is to learn the true density $g_{\star}$ under minimal assumptions from samples of $\Pa$ and $\Q$. To this end,  we assume that
\begin{equation}\label{assuginL2}
 g_{\star} \in L^2_{\Pa},
\end{equation}
where we write $L^2_\M = L^2(\Zcal,\M)$ for the separable $L^2$-space of a.s.-equivalence classes of real-valued functions on $\Zcal$, and where $\M$ denotes a placeholder for either $\Pa$ or $\Q$.

We let $\Hcal$ be a separable reproducing kernel Hilbert space (RKHS) with bounded measurable reproducing kernel $k$ on $\Zcal$,
\begin{equation}\label{asskappa}
  \kappa  \coloneqq \sup_{z\in\Zcal} k(z,z)^{1/2}<\infty,
\end{equation}
see, e.g., \citet{ste_sco_12, pau_rag_16} for the definition and properties of an RKHS.\footnote{By the kernel property, Assumption \eqref{asskappa} implies that $|k(z_1,z_2)|  \le \sqrt{k(z_1,z_1) k(z_2,z_2)}$, hence $ \sup_{z_1,z_2\in\Zcal} |k(z_1,z_2)|=\sup_{z\in\Zcal}k(z,z)$.}
Hence all functions $h\in\Hcal$ are measurable and bounded, and the canonical embeddings $J_\M    \colon \Hcal\to L^2_{\M}$ that map $h$ onto its respective a.s.-equivalence class $J_\M h$, are Hilbert--Schmidt operators with adjoints
\begin{align*}
       J_\M^\ast f   = \int_{\Zcal}
  k(\cdot,z)f(z)\, \M(dz) ,  \quad  f\in L^2_{\M},
\end{align*}
see, e.g., \citet[Section 2]{ste_sco_12}. 

We consider candidate densities of the form
\begin{equation}\label{eq:ansatz}
g = p + J_\Pa h, \quad h\in\Hcal,
\end{equation}
where $p:\Zcal\to\R$ is a bounded measurable prior function, 
\begin{equation}\label{asspi}
\|p\|_\infty \coloneqq \sup_{z\in\Zcal} |p(z)| < \infty.
\end{equation}
The additive decomposition \eqref{eq:ansatz} is motivated by the fact that $g_\star$, even if bounded, need not belong to $\Hcal$, whereas $g_\star - p$ may. To illustrate this, consider the case $\Zcal=\R$, $g_\star=1$, and $k$ a Gaussian kernel. The associated RKHS does not contain constant functions (indeed, it contains no polynomials; see \citet{quang10}), so choosing the constant prior $p=1$ makes the residual representable in $\Hcal$. In general, we neither assume that $p$ is positive nor that it integrates to one with respect to $\Pa$, so that $p=0$ is always admissible. Notably, the specification \eqref{eq:ansatz} appears to be new in the literature and accommodates a broader class of densities than the structural assumptions typically imposed in prior work. For a detailed comparison, see Section~\ref{ssec_comp}.

\medskip
\noindent\textbf{Standing Assumption.}
\emph{The above standing assumptions, in particular \eqref{assu1}--\eqref{asspi}, apply throughout the paper, even when not explicitly restated.}
\medskip

It remains to introduce a suitable error function $\Ecal(h)$, measuring the distance between the true $ \Q=g_\star  \Pa$ and the candidate measure $({p}+J_\Pa h) \Pa$, that can be empirically evaluated for any $h\in\Hcal$. To this end we consider the squared worst-case expectation error over all normalized test functions \(f\in L^2_{\Pa}\), 
\begin{equation}\label{eqPML2}
\begin{aligned}
  \Ecal(h)&\coloneqq \sup_{\|f\|_{L^2_{\Pa}}\le 1}
 \Big(\int_{\Zcal} f(z)\,\Q(d z)
 - \int_{\Zcal} f(z)({p}(z)+  h(z)) \, \Pa(d z)\Big)^2 \\
 &  = \sup_{\|f\|_{L^2_{\Pa}}\le 1}
  \langle f, g_{\star} - {p}- J_\Pa h\rangle_{L^2_{\Pa}}^2
 = \|  g_{\star} - {p} - J_\Pa h\|_{L^2_{\Pa}}^2.
\end{aligned}
 \end{equation}

Adding Tikhonov regularization through the penalty term $\lambda \|h\|_{\Hcal}^2$ for some $\lambda> 0$, we arrive at the convex problem
\begin{equation}\label{optUCpre}
  \underset{h\in\Hcal}{\text{minimize}} \,  \big\{ \Ecal(h)  +  \lambda \|h\|_\Hcal^2\big\}.
\end{equation}
\begin{equation}\label{hlambdaeqpre}
 h_\lambda = ( J_\Pa^\ast J_\Pa + \lambda)^{-1} J_\Pa^\ast( g_\star-  {p} ) .
\end{equation}
This standard result follows from, e.g., \citet[Theorem 5.1]{eng_et_al_96}. In contrast, a solution to \eqref{optUCpre} does not always exist for the limit case $\lambda=0$, as the following remark clarifies.

\begin{remark}
 The operator $J_\Pa^\ast J_\Pa$ on $\Hcal$ is trace-class. It is invertible if and only if $\ker J_\Pa=\{0\}$ and $\dim \Hcal<\infty$. In general, the right hand side of~\eqref{hlambdaeqpre} is thus not well defined for $\lambda=0$.
\end{remark}

The error function $\Ecal(h)$, and the expression in \eqref{hlambdaeqpre}, contain the unobservable population density $g_{\star}$ and therefore are not readily available for estimations based on samples of $\Pa$ and $\Q$. We address this problem by exploiting the identity $J_\Pa^\ast g_\star = J_\Q^\ast 1$, justified by the following lemma, whose proof is provided in Appendix~\ref{app:proof_lemma2}.

\begin{lemma}\label{lemJast0Jast}
We have $J_\Pa^\ast (f\trueg) = J_\Q^\ast f $ for any bounded measurable function $f$ on $\Zcal$.
\end{lemma}

We infer that $\Ecal(h)=\|  g_{\star} - {p} \|_{L^2_{\Pa}}^2 - 2\langle J_\Q^\ast 1 - J_\Pa^\ast {p},h\rangle_\Hcal + \langle J_\Pa^\ast J_\Pa h,h\rangle_\Hcal $, and therefore problem \eqref{optUCpre} is equivalent to the bona fide convex problem
\begin{equation}\label{optUCNS}
\underset{h\in\Hcal}{\text{minimize}} \, \big\{ - 2 \langle J_\Q^\ast 1 - J_\Pa^\ast {p}, h\rangle_\Hcal + \langle (J_\Pa^\ast J_\Pa +\lambda) h, h\rangle_\Hcal  \big\},
\end{equation}
and its solution \eqref{hlambdaeqpre} can be represented as
\begin{equation}\label{hlambdaeqNS}
 h_\lambda = ( J_\Pa^\ast J_\Pa + \lambda)^{-1} (J_\Q^\ast 1- J_\Pa^\ast {p} ) .
\end{equation}

To assess how close ${p}+ J_\Pa h_\lambda$ is  to the true density $g_\star$,  we consider the orthogonal decomposition of the population error
 \begin{equation}\label{SQAE}
   \Ecal(h_\lambda)=\|   g_\star - {p}-J_\Pa h_\lambda \|_{L^2_{\Pa}}^2 =   \underbrace{\| g_\star - g_0  \|_{L^2_{\Pa}}^2}_{\text{projection error}} + \underbrace{\| g_0 -{p}- J_\Pa h_\lambda  \|_{L^2_{\Pa}}^2}_{\text{regularization error}} ,
  \end{equation}
into the sum of the squared projection error and squared regularization error, where $g_0-{p}$ denotes the orthogonal projection of $g_{\star}-{p}$ onto the closure $\overline{\Ima J_\Pa}$ of the image of $J_\Pa$ in $L^2_{\Pa}$. The following lemma collects some elementary facts about these error terms, the proof is provided in Appendix~\ref{app:proof_lemma3}.

\begin{lemma}\label{lemregerror0} 
\begin{enumerate}

\item\label{lemregerror00} The projection error vanishes, $\|   g_\star - g_0 \|_{L^2_{\Pa}}=0$, if and only if the ground truth $g_\star - {p}$ lies in $\overline{\Ima J_\Pa}$. This holds in particular if the kernel $k$ is universal in the sense that $\overline{\Ima J_\Pa}=L^2_{\Pa}$.

    \item\label{lemregerror01} The regularization error vanishes as $\lambda$ tends to zero, $\lim_{\lambda\to 0} \| g_0-{p}-J_\Pa h_\lambda \|_{L^2_{\Pa}} = 0$, albeit the convergence may be slow in general.

\item\label{lemregerror02} Well-posedness: Problem \eqref{optUCpre} admits a solution for $\lambda=0$ if and only if the projection 
\begin{equation}\label{asswellposed}
\text{$g_0 - {p} = J_\Pa h_0$ is attained for some $h_0\in \Hcal$.}
\end{equation}

\item\label{lemregerror03} Source condition: In addition to \eqref{asswellposed} assume that  
\begin{equation}\label{sourcecond}
\text{$h_0 \in \Ima (J_\Pa^\ast J_\Pa)^{{\nu}}$ for some ${{\nu}}\in [0,1/2]$. }
\end{equation}
Then 
\begin{align}
     \| (J_\Pa^\ast J_\Pa  +\lambda)^{-{{\nu}}} h_\lambda\|_\Hcal \le \| (J_\Pa^\ast J_\Pa)^{-{{\nu}}} h_\lambda\|_\Hcal &\le \|(J_\Pa^\ast J_\Pa)^{-{\nu}} h_0\|_\Hcal, \label{hb1s}\\
      \lim_{\lambda\to 0}  (J_\Pa^\ast J_\Pa+\lambda)^{-{{\nu}}} h_\lambda  = \lim_{\lambda\to 0} (J_\Pa^\ast J_\Pa)^{-{{\nu}}} h_\lambda&=(J_\Pa^\ast J_\Pa)^{-{\nu}} h_0\quad\text{in $\Hcal$,}  \label{hb2s}
\end{align}
and the following rates of convergence in $\lambda$ hold, 
\begin{align}
    \| h_\lambda-h_0\|_\Hcal &\le  (1-{{\nu}})  \|(J_\Pa^\ast J_\Pa)^{-{\nu}} h_0\|_\Hcal  \lambda^{{{\nu}}} ,\label{hb3s}\\
    \| g_0-{p}-J_\Pa h_\lambda \|_{L^2_{\Pa}}  &\le  ({{\nu}}+1/2)  \|(J_\Pa^\ast J_\Pa)^{-{\nu}} h_0\|_\Hcal  \lambda^{{{\nu}}+1/2} .\label{hb4s}
\end{align}
 
\end{enumerate}
\end{lemma}

\section{Sample estimator, asymptotic guarantees, and representer theorem}\label{sec:sampleestimator}

To facilitate empirical work, we next introduce sample estimators for $h_\lambda$. 
Let $z_{\Pa,1},\dots,z_{\Pa,n}$ and $z_{\Q,1},\dots,z_{\Q,n}$ be i.i.d.~samples from $\Pa$ and $\Q$, respectively, with sample size $n\ge 1$.\footnote{For notational simplicity, we assume samples of equal size $n$. All results below can be extended to differing sample sizes for $\Pa$ and $\Q$, albeit at the cost of additional notational complexity.} 
We denote the combined sample by
\[
z_i\coloneqq z_{\Pa,i},\quad 
z_{n+j}\coloneqq z_{\Q,j},
\quad\text{for } i,j=1,\dots,n ,
\]
and define the arrays of sample points $\bm z\coloneqq [z_1,\dots,z_{2n}]^\top$, and similarly $\bm z_{\Pa}$ and $\bm z_{\Q}$.

For any function $h$ on $\Zcal$ and any array of points $\bm\xi=[\xi_1,\dots,\xi_t]^\top$, we write $h(\bm\xi)$ for the array of function values and $k(\cdot,\bm\xi^\top)=[k(\cdot,\xi_1),\dots,k(\cdot,\xi_t)]$ for the corresponding array of kernel functions.

We define the corresponding sample operators $S_\M:\Hcal\to \R^n$ by
\[
S_\M h \coloneqq h(\bm z_\M),
\]
which are the sample analogues of $J_\M$. Their adjoints are $S_\M^\ast v = k(\cdot,\bm z_\M^\top) v $. Note that the sample analogue of $J_\M^\ast$ is $ n^{-1} S_\M^\ast$. Hence the sample analogue of \eqref{optUCNS} is the convex problem
\begin{equation}\label{optUCpsample}
   \underset{h\in\Hcal}{\text{minimize}} \,  \big\{ - 2 \langle S_\Q^\ast \bm 1 - S_\Pa^\ast \bm {p}, h\rangle_\Hcal + \langle (S_\Pa^\ast S_\Pa +n \lambda) h, h\rangle_\Hcal  \big\},
\end{equation}
where we define the column vectors of ones $\bm 1\coloneqq [1,\dots,1]^\top$ and function values $\bm {p} \coloneqq p(\bm z_{\Pa})$. The unique solution of \eqref{optUCpsample} can be calculated to be
\begin{equation}\label{hlambdaeqsample} 
\hat h_\lambda = (S_\Pa^\ast S_\Pa + n\lambda)^{-1} \big(S_\Q^\ast \bm 1 - S_\Pa^\ast \bm {p} \big),
\end{equation}
which is our sample estimator of \eqref{hlambdaeqNS}.

We first provide a functional central limit theorem for this estimator.\footnote{In Proposition~\ref{propFSGaux} in the appendix we also provide corresponding finite-sample guarantees.} In the next section, we then show how to efficiently compute it. For any bounded measurable function $g:\Zcal\to\R$, we define the bounded operator $\diag(g): L^2_\M\to L^2_\M$ by point-wise multiplication, $\diag(g) f\coloneqq g \cdot f$, which is a natural generalization of the vector-to-matrix $\diag$ operator.

\begin{theorem}\label{thmAC}
The estimator \eqref{hlambdaeqsample} satisfies the following properties, 
    \begin{enumerate}
  \item\label{thmAC1} Asymptotic consistency: $\hat h_\lambda \to h_\lambda$ in $\Hcal$ a.s.\ as $n\to\infty$.

  \item\label{thmAC2} Functional central limit theorem: $n^{1/2}(\hat h_\lambda - h_\lambda)\to \Ncal(0,O_\lambda)$ in distribution as $n\to\infty$, where the covariance operator $O_\lambda:\Hcal\to\Hcal$ is given by $O_\lambda = (J_\Pa^\ast J_\Pa +\lambda)^{-1} Q_\lambda (J_\Pa^\ast J_\Pa +\lambda)^{-1}$ for the non-negative self-adjoint trace-class operator $Q_\lambda:\Hcal\to\Hcal$ defined by the sum
  \begin{equation}\label{Clamdef}
     \begin{aligned}
  Q_\lambda  &\coloneqq J_\Q^\ast J_\Q - (J_\Q^\ast 1)\otimes (J_\Q^\ast 1) \\
  &\quad + J_\Pa^\ast \diag({p}+J_\Pa h_\lambda)^2 J_\Pa - (J_\Pa^\ast({p}+J_\Pa h_\lambda))\otimes (J_\Pa^\ast({p}+J_\Pa h_\lambda))
\end{aligned} 
  \end{equation}
  of the covariance operators of the canoncial feature map $z\mapsto k(\cdot,z)$ under the true measure $\Q$ and the population estimate $({p}+J_\Pa h_\lambda)\Pa$.

\end{enumerate}
\end{theorem}
The proof is provided in Appendix~\ref{app:proof_thm4}.
Theorem~\ref{thmAC} and its proof will provide a basis for the two-sample tests below.

We now derive a representer theorem for $\hat h_\lambda$. To this end, we define the joint sampling operator
\[
S:\Hcal\to \R^{2n}\cong\R^n\oplus \R^n, 
\quad 
S\coloneqq 
\begin{bmatrix}
S_\Pa\\ 
S_\Q
\end{bmatrix},
\]
whose adjoint is
\[
S^\ast =
\begin{bmatrix}
S_\Pa^\ast ,\; S_\Q^\ast
\end{bmatrix}
:\R^{2n}\to\Hcal .
\]
Left-multiplying \eqref{hlambdaeqsample} by $S_\Pa^\ast S_\Pa + n\lambda$ shows that $\hat h_\lambda$ lies in the finite-dimensional subspace $\Ima S^\ast$ of $\Hcal$. The following lemma provides its coordinate representation in terms of the symmetric positive semidefinite $(2n)\times(2n)$ kernel matrix
\[
\bm K=
\begin{bmatrix}
\bm K_\Pa & \bm K_{\Pa\Q} \\
\bm K_{\Q\Pa} & \bm K_\Q
\end{bmatrix}
\coloneq k(\bm z,\bm z^\top),
\]
which is the matrix representation of the operator $SS^\ast$ on $\R^{2n}$. We denote by $\bm A^+$ the pseudoinverse of a matrix $\bm A$.

\begin{lemma}\label{lemRT}
The sample estimator \eqref{hlambdaeqsample} admits the coordinate representation
\begin{equation}\label{eqRT}
\hat h_\lambda =
k(\cdot,\bm z^\top)
\big(\bm K_{:,\Pa}\bm K_{\Pa,:} + \lambda n \bm K \big)^+
\big(\bm K_{:,\Q}\bm 1 - \bm K_{:,\Pa}\bm p\big).
\end{equation}
\end{lemma}
See Appendix~\ref{app:proof_lemma5} for the derivation.

\section{Nystr\"om approximation and optimal error rates}\label{sec:lowrank}

The $(2n)\times(2n)$ pseudoinverse in \eqref{eqRT} can in principle be computed, but for large $n$, say $n\ge 10^5$, the computational cost becomes prohibitive.

A standard approach to mitigate this issue is the Nystr\"om method, which selects a subsample of points $z_{\pi_1},\dots,z_{\pi_m}$ corresponding to an index set 
$\Pi=\{\pi_1,\dots,\pi_m\}\subseteq\{1,\dots,2n\}$ of size $m \le 2n$, and approximates the kernel $k$ by projecting it onto the subspace spanned by the corresponding kernel basis functions,
\[
\Hcal_\Pi\coloneqq\spn\{ k(\cdot,z_{\pi_1}),\dots, k(\cdot,z_{\pi_m})\},
\]
in $\Hcal$. The Nystr\"om approximation is well known in the machine learning literature; see, e.g., \citet[Section 19.2]{mar_tro_20} or \citet{che_etal_25}. This literature also discusses pivoting schemes, see \citet{hor_zha_05}, \citet{HPS12}, and \citet{filipovic2024adaptive}. 
 
Define $\ell\coloneq\rank(\bm K_{\Pi,\Pi})\le m$, and let $\bm R$ be any $m\times \ell$ matrix such that $\bm R\bm R^\top=(\bm K_{\Pi,\Pi})^+$, and define the $(2n)\times \ell$ matrix
\begin{equation}\label{Ldef}
    \bm L=
\begin{bmatrix}
\bm L_\Pa \\
\bm L_\Q
\end{bmatrix}
\coloneq
\bm K_{:,\Pi}\bm R,
\end{equation}  
so that $\bm L\bm L^\top=\bm K_{:,\Pi}(\bm K_{\Pi,\Pi})^+\bm K_{\Pi,:}$
is the column Nystr\"om approximation of $\bm K$, see, e.g., \citet[Section 19.2]{mar_tro_20}. In practice, $\bm L$ could be given by a pivoted incomplete Cholesky decomposition of $\bm K$, which along with $\bm R$ can be efficiently computed by incremental algorithms, see \citet{filipovic2024adaptive}.\footnote{In practice, $\bm K_{\Pi,\Pi}$ can be ill-conditioned, and we therefore compute its inverse numerically via a truncated pseudoinverse: eigenvalues smaller than $\tau$ times the largest eigenvalue are set to zero, for some threshold $\tau>0$ (e.g., $\tau=10^{-12}$). This is particularly relevant for the Gaussian kernel, whose eigenvalues decay rapidly.} It is an elementary fact that the orthogonal projection $P_\Pi$ in $\Hcal$ onto $\Hcal_\Pi$ is given by
\begin{equation}\label{R1}
P_\Pi h =
k(\cdot,\bm z_\Pi^\top)\bm R\bm R^\top
h(\bm z_\Pi),
\quad h\in\Hcal .
\end{equation}
Hence $P_\Pi h$ can be evaluated efficiently by querying only $m$ kernel basis functions.

We approximate the convex problem \eqref{optUCpsample} over $\Hcal$ by restricting it to the subspace $\Hcal_\Pi$. This yields the $m$-dimensional convex problem
\begin{equation}\label{optUCpsampleLR}
\underset{h\in\Hcal_\Pi}{\text{minimize}}\;
\big\{
- 2 \langle P_\Pi (S_\Q^\ast \bm 1 - S_\Pa^\ast \bm p), h\rangle_\Hcal
+
\langle (P_\Pi S_\Pa^\ast S_\Pa P_\Pi + n \lambda) h, h\rangle_\Hcal
\big\}.
\end{equation}
The unique solution of \eqref{optUCpsampleLR} is
\begin{equation}\label{hlambdaeqsampleLR}
\hat h_{\lambda,\Pi}
=
(P_\Pi S_\Pa^\ast S_\Pa P_\Pi + n\lambda)^{-1}
P_\Pi ( S_\Q^\ast \bm 1 - S_\Pa^\ast \bm p ),
\end{equation}
which is a low-rank approximation of the sample estimator \eqref{hlambdaeqsample}. It can be computed efficiently, as the following result shows.

\begin{lemma}\label{lemapplocal}
The approximated estimator \eqref{hlambdaeqsampleLR} can be expressed in coordinates as
\begin{align}  
\hat h_{\lambda,\Pi}
&=
k(\cdot,\bm z_\Pi^\top)\bm R
\big(\bm L_\Pa^\top \bm L_\Pa + n\lambda\big)^{-1}
\big(\bm L_\Q^\top \bm 1 - \bm L_\Pa^\top \bm p\big)\label{hlambdaeqsampleLRcoor1}\\
&=k(\cdot,\bm z_\Pi^\top)\big(\bm K_{\Pi,\Pa}\bm K_{\Pa,\Pi} + \lambda n \bm K_{\Pi,\Pi}\big)^+\big(
\bm K_{\Pi,\Q}\bm 1 -\bm K_{\Pi,\Pa}\bm p\big),\label{hlambdaeqsampleLRcoor2}
\end{align}
which can be computed with complexity $\Ocal(m^2 n)$.
\end{lemma}
See Appendix~\ref{app:proof_lemma6} for the derivation. Expression~\eqref{eqRT} is a special case of \eqref{hlambdaeqsampleLRcoor2} with $m=2n$.

We next discuss schemes for selecting the set of Nystr\"om points $\Pi$. In particular, we consider (random) $\Pa$-schemes that select Nystr\"om points exclusively from the $\Pa$-sample,
\begin{equation}\label{Ppivot}
\Pi=\{i_1,\dots,i_m\}\subseteq\{1,\dots,n\}.
\end{equation}
While there is no a priori reason why this restriction should outperform subsampling from the combined sample, the following theorem shows that it nevertheless achieves optimal total error rates for the low-rank approximation. The key observation is that the computational error
\begin{equation}\label{defCE}
\Ccal_{\lambda,\Pi}\coloneq \|(I-P_\Pi) (J_\Pa^\ast J_\Pa+\lambda)^{1/2}\|^2,
\end{equation}
plays a central role: although defined for an arbitrary set of points $\Pi$, it can be controlled solely in terms of the number of Nystr\"om points drawn from the $\Pa$-sample.

\begin{theorem}\label{thmAEbound}
Assume that source condition \eqref{sourcecond} holds for some $\nu\in [0,1/2]$, and denote by $R\coloneq \|(J_\Pa^\ast J_\Pa)^{-\nu}h_0\|_\Hcal < \infty$.\footnote{This condition corresponds to \cite[Assumptions 1 and 4]{rud_cam_ros_15}.} 
Let the set $\Pi$ be given by any (random) $\Pa$-scheme \eqref{Ppivot} such that, for any $s\in (0,1)$ with subsampling probability of at least $1-s$, the computational error is bounded by
\begin{equation}\label{asscomperror}
\Ccal_{\lambda,\Pi} \le 3\lambda,\quad\text{if $m\ge C_{CE}(s,\lambda,n)$,}
\end{equation}
for a coefficient $C_{CE}$ that grows in $\lambda^{-1}$ and $n$ at rate 
\begin{equation}\label{rateCCE}
C_{CE}(s,\lambda,n)=O(\lambda^{-1}\log n).
\end{equation}  
Let $\eta\in (0,1)$ and assume
\begin{equation}\label{lambdancondn}
 n^{(1+2\nu)/(2+2\nu)} \ge \frac{13 \kappa^2}{\|J_\Pa^\ast J_\Pa\|}\big( \log\frac{24\kappa^2}{\eta }+\frac{1}{2+2\nu }\log(n)\big).
\end{equation}
Then the following bounds hold jointly with sampling probability of at least $1-\eta$,
\begin{align}
\|  {\hat h_{\lambda,\Pi}}- h_0 \|_\Hcal  &\le \big(C_{AE,\Hcal}(\eta,\nu) + (1-\nu)R \|J_\Pa^\ast J_\Pa\|^\nu\big) n^{-\nu/(2+2\nu)},\label{eqHrates}\\
  \Ecal({\hat h_{\lambda,\Pi}})  -  \Ecal(h_0)  &\le \big(   C_{AE,L^2_\Pa}(\eta,\nu) + (1/2+\nu)R \|J_\Pa^\ast J_\Pa\|^{1/2+\nu}\big)^2   n^{-(1+ 2\nu)/(2+2\nu)},\label{totalrates}
\end{align}
with $ \hat h_{\lambda,\Pi}$ as in \eqref{hlambdaeqsampleLRcoor1}--\eqref{hlambdaeqsampleLRcoor2}, $\lambda = \| J_\Pa^\ast J_\Pa\| n^{-1/(2+2\nu)}$ and
\begin{equation}\label{eqmLB}
m \ge C_{CE}(\eta/3,\| J_\Pa^\ast J_\Pa\| n^{-1/(2+2\nu)},n) = O(n^{1/(2+2\nu)}\log(n)),
\end{equation}
for the coefficients
\begin{align}
C_{AE,\Hcal}(\eta,\nu) &\coloneq 3^{1+\nu}  R \|J_\Pa^\ast J_\Pa\|^\nu    + 1.16\, C_{\Delta}(\eta)  ,\label{defCAEH}\\
C_{AE,L^2_\Pa}(\eta,\nu) &\coloneq \big(4.19\cdot 3^{\nu}   R \|J_\Pa^\ast J_\Pa\|^{\nu}   + 1.42\,  C_{\Delta}(\eta)   \big)\|J_\Pa^\ast J_\Pa\|^{1/2} ,\label{defCAEL} 
\end{align}
and $C_\Delta(\eta)\coloneqq 2\kappa\sqrt{2\log(12/\eta)}\big(  1+\|p\|_\infty +R  \kappa^3\big)\big( 1 +0.07 \eta^{-1/2}\big)\|J_\Pa^\ast J_\Pa\|^{-1}$.
\end{theorem}
The proof of the theorem is provided in Appendix~\ref{app:proof_thm}.

A baseline scheme satisfying the assumptions of Theorem~\ref{thmAEbound} is described in the following remark.

\begin{remark}
The \textbf{plain Nystr\"om $\Pa$-scheme} selects $m$ indices uniformly at random from $\{1,\dots,n\}$ without replacement. By \citet[Lemma 6]{rud_cam_ros_15}, this scheme satisfies \eqref{asscomperror}--\eqref{rateCCE} with coefficient
\[
C_{CE}(s,\lambda,n)=  (67\vee 5\kappa^2\lambda^{-1})\log\big(4\kappa^2/(\lambda s)\big).
\]
Another admissible $\Pa$-scheme considered in \cite{rud_cam_ros_15} is the approximate leverage scores Nystr\"om subsampling, for which an explicit expression for $C_{CE}(s,\lambda,n)$ is also available. It remains an open problem to establish comparable, potentially sharper bounds for alternative schemes, such as the adaptive random pivoting scheme of \citet{che_etal_25}.
\end{remark}

The total error rates in \eqref{totalrates} compare to the main result in \cite[Theorem 1]{rud_cam_ros_15}, who also show that these error rates are optimal in a minimax sense.\footnote{Our rates correspond to the rates in \cite[Theorem 1]{rud_cam_ros_15} for their ``$\gamma=1$''. } Comparing to  \citet[][Section 2.7]{liracine06} for local kernel density estimators on Euclidean state spaces $\Zcal\subseteq\R^d$ of dimension $d$, their convergence rate of the mean squared error is quoted as $O(n^{ -4/(d+4) })$, with correctly specified model and optimal bandwidth.  As noted in \citet{DUONG2005417}, local kernel density estimation is therefore deemed challenging for dimensions $d>4$, even with large samples.

\section{Applications}\label{sec:applications} 
In this section, we illustrate  two applications of KDM. The first is hypothesis testing, the second is estimating conditional distributions.

\subsection{Kernel-based two-sample testing}
\label{sec:two-sample}

As a first application, we obtain statistical tests of the hypothesis that the true equals the prior density,
\begin{equation}\label{nullh}
    g_\star={p},\quad \text{$\Pa$-a.s.}
\end{equation}  
The following theorem provides testable implications of KDM.\footnote{Under the assumption that the kernel $k$ is universal in the sense that $\overline{\Ima J_\Pa}=L^2_\Pa$, or equivalently $\ker J_\Pa^\ast=\{0\}$, the converse of Theorem~\ref{lthmnullh} also holds: \eqref{nullh1} implies \eqref{nullh}. This follows from~\eqref{hlambdaeqpre}. On the other hand, \eqref{nullh} implies that $g_0={p}$. However, the converse does not hold, because the orthogonal projection of $g_\star-{p}$ onto $\overline{\Ima J_\Pa}$ can be zero also when \eqref{nullh} does not hold.}

\begin{theorem}\label{lthmnullh}
Hypothesis \eqref{nullh} implies that
\begin{equation}\label{nullh1}
    h_\lambda=0,\quad\text{for any $\lambda>  0$.}
\end{equation}
In this case, the following testable properties hold, where $\bm L$ is given by \eqref{Ldef}:
\begin{enumerate}
    \item\label{lthmnullh1} Let $\lambda>0$ and $\eta\in (0,1)$. Then, with sampling probability of at least $1-\eta$, we have
    \begin{multline}\label{eqtestnullnew}
     \|   \big(\bm L_\Pa^\top \bm L_\Pa + n\lambda\big)^{-1}  \big(\bm L_\Q^\top \bm 1 - \bm L_\Pa^\top \bm {p} \big)\|_2    \\ \le  1.16\cdot 2\kappa\sqrt{2\log(12/\eta)}\big(1+\|p\|_\infty\big)\big( 1 +0.07\eta^{-1/2}\big) \lambda^{-1} n^{-1/2},
 \end{multline}
 for all $n$ such that $\lambda n\ge 13\kappa^2\log\frac{24\kappa^2}{\eta\lambda}$, and asymptotically for large $n$,
 \begin{multline} \label{eqtestnullJP}
     \|  \bm L_\Pa \big(\bm L_\Pa^\top \bm L_\Pa + n\lambda\big)^{-1}  \big(\bm L_\Q^\top \bm 1 - \bm L_\Pa^\top \bm {p} \big)\|_2    \\ \le  \sqrt{2}\cdot 2\kappa\sqrt{2\log(12/\eta)}\big(1+\|p\|_\infty \big)\big( 1 +0.07\eta^{-1/2}\big) \lambda^{-1/2}.
 \end{multline}

\item\label{lthmnullh2} Asymptotically for large $n$, the $\R^\ell$-valued sample variable  \begin{equation}\label{vlamN}
    v_\lambda \coloneqq n^{-1/2}(\bm L_\Q^\top\bm 1-\bm L_\Pa^\top\bm {p}) \sim \Ncal(0,\bm\Sigma)
\end{equation}  
is normally distributed with mean zero and $\ell\times \ell$-covariance matrix given by
\begin{equation}\label{SigN}
    \bm \Sigma = n^{-1}\bm L_\Q^\top \bm L_\Q - n^{-2} \bm L_\Q^\top\bm 1\bm 1^\top \bm L_\Q + n^{-1} \bm L_\Pa^\top \diag(\bm {p})^2\bm L_\Pa - n^{-2} \bm L_\Pa^\top\bm {p}\bm {p}^\top \bm L_\Pa.
\end{equation}

\item\label{lthmnullh3} Plug-in Mahalanobis distance: Asymptotically for large $n$, the test statistic
\begin{equation}\label{eq:teststat} 
T_\ell\coloneqq v_\lambda^\top \bm \Sigma^{-1} v_\lambda \sim\chi^2(\ell)
\end{equation}
is $\chi^2$-distributed with $\ell$ degrees of freedom.

\item\label{lthmnullh4} Plug-in Gamma approximation: Denote by $\bm\Sigma = \bm A \bm W\bm A^\top$ the spectral decomposition of the matrix in \eqref{SigN}, with normalized eigenvectors $\bm A=[a_1,\dots,a_m]$ and eigenvalues $w_1\ge \cdots\ge w_\ell> 0$ on the diagonal of $\bm W$. Then, asymptotically for large $n$, the test statistic
\begin{equation}
    S_\Gamma \coloneq \| \bm A^\top v_\lambda\|_2^2 \sim \Gamma(\alpha,\beta)
\end{equation}
is approximately Gamma-distributed with shape $\alpha=\frac{\trace(\bm W)^2}{2\trace(\bm W^2)}$ and scale $\beta=\frac{2\trace(\bm W^2)}{\trace(\bm W)}$.
\end{enumerate}
\end{theorem}
The proof of the theorem is provided in Appendix~\ref{app:proof_thm9}.

For ${p}=1$, the hypothesis \eqref{nullh} is equivalent to $\Q=\Pa$, which is the same hypothesis as studied in \citet{pfisteretal17}. Then the sample variable $u_\lambda$ in \eqref{ulamhyp} in the proof of Theorem~\ref{lthmnullh} is the scaled difference $S_\Q^\ast \bm 1 - S_\Pa^\ast \bm 1$ between the sample kernel mean embeddings of $\Q$ and $\Pa$. As a corollary of Theorem~\ref{lthmnullh} we thus recover the kernel two-sample test of \citet{gre_etal_12} as a special case, connecting their maximum mean discrepancy approach \citep[][Theorem 5]{gre_etal_12} with the Hilbert--Schmidt independence criterion \citep[][Proposition 1]{pfisteretal17}. The Gamma approximation in \ref{lthmnullh4} has been suggested by \citet{gre_bou_smo_sch_05,pfisteretal17}. It remedies the situation where the test statistic \eqref{eq:teststat} is numerically unstable for eigenvalues $w_i$ that are close to zero.

\subsection{Conditional distribution estimation}
\label{sec:cond-distr-est}

As a second application, we obtain a novel, consistent estimator of conditional distributions. Thereto, we assume that $\Zcal=\Xcal\times\Ycal$ is the product of two countably generated measurable spaces $\Xcal$ and $\Ycal$, and consider random variables $X$ and $Y$ with values in $\Xcal$ and $\Ycal$. We denote by $\Pa_X$, $\Pa_Y$ and $\Pa_{(X,Y)}$ their marginal and joint distributions, and set $\Pa\coloneqq\Pa_X\otimes\Pa_Y$ and $\Q\coloneqq\Pa_{(X,Y)}$. Assumption \eqref{assu1} then reads
\begin{equation}\label{assu1XY}
    \Pa_{(X,Y)}\ll \Pa_X\otimes\Pa_Y.
\end{equation}
We first derive our estimator of the conditional distribution based on a joint sample of $(X,Y)$. We then briefly discuss the relevance and restrictiveness of Assumption~\eqref{assu1XY}. We also relate our conditional distribution estimator to the literature. 

\subsubsection{Derivation of the estimator}

Assumption~\eqref{assu1XY} implies that the conditional distribution of $Y$ given
$X=x$, and of $X$ given $Y=y$, exists and (a version of it) can be expressed in terms of $g_{\star}$ as
\begin{equation}\label{condist}
 \Pa_{Y\mid X=x}( d  y) = g_{\star}(x,y)\Pa_Y( d  y),
\end{equation}
and $\Pa_{X\mid Y=y}( d  x) = g_{\star}(x,y)\Pa_X( d  x)$, respectively. In the following, we will concentrate on the former. By symmetry, all statements also apply to the latter.

For any candidate density \eqref{eq:ansatz} with $h\in\Hcal$, we thus obtain the corresponding model of the true conditional expectation $\E
[f(Y)\mid X=x]=\int_\Ycal f (y) \Pa_{Y\mid  X=x}(d y)$ given by
\begin{equation}\label{eqCEfYX}
     \E_h[ f(Y)\mid X=x]\coloneqq \int_\Ycal f (y) ( {p}(x,y)+h(x,y))\,\Pa_Y(dy),\quad \text{for $f\in  L^2_{\Pa_Y}$.}
\end{equation}
In view of the isometric isomorphisms
\begin{equation}\label{eqisoL2}
  L^2_{\Pa_X\otimes \Pa_Y}=L^2_{\Pa_X}\otimes L^2_{\Pa_Y}=
  L^2_{\Pa_X}(\Xcal;L^2_{\Pa_Y}) =L^2_{\Pa_Y}(\Ycal;L^2_{\Pa_X}),
\end{equation}
the right hand side of \eqref{eqCEfYX} as a function of $x$ is an element in $L^2_{\Pa_X}$. In that sense, we can decompose the error function \eqref{eqPML2} as
\begin{equation}\label{boundEYXnew}
\begin{aligned}
    \Ecal(h)  &= \int_\Xcal \int_\Ycal \big(g_\star(x,y) - {p}(x,y)-h(x,y)\big)^2\,\Pa_Y(dy) \Pa_X(d x) \\
     &=  \int_\Xcal  \bigg(\sup_{\|f\|_{L^2_{\Pa_Y}}\le 1}\Big|\E
[f(Y)\mid X=x]-\E_h
[f(Y)\mid X=x]\Big|\bigg)^2\, \Pa_X(d x).
\end{aligned}
 \end{equation}
Hence $\Ecal(h)$ can be interpreted as the squared worst-case error of the conditional expectation model \eqref{eqCEfYX}. The guarantees in Theorem~\ref{thmAEbound} apply accordingly to the following low-rank estimator of the conditional expectation,\footnote{In the sense that for any test function $f\in L^2_{\Pa_Y}$ the squared $L^2_{\Pa_X}$-error of the conditional expectation is bounded by $ \big\|\E
[f(Y)\mid X=\cdot]-\E_{{\hat h_{\lambda,\Pi}}}
[f(Y)\mid X=\cdot] \big\|_{L^2_{\Pa_X}}^2  \le \Ecal({\hat h_{\lambda,\Pi}}) \|f\|_{L^2_{\Pa_Y}}^2$, which can be bounded by \eqref{totalrates}.} which we obtain by combining \eqref{eqCEfYX} with the previous sections, 
\begin{align}
\E_{{\hat h_{\lambda,\Pi}}}
[f(Y)\mid X=x] &=\int_\Ycal f (y) \big( {p}(x,y)+{\hat h_{\lambda,\Pi}}(x,y)\big)\,\Pa_Y(dy)\label{approxCE1}\\
  &\approx  \bar{n}^{-1} \sum_{i=1}^{\bar n}  f(\bar y_i)\big({p}(x,\bar y_i)  +   {\hat h_{\lambda,\Pi}}(x,\bar y_i)\big)\label{approxCE2}\\
  &\approx\frac{ \sum_{i=1}^{\bar n}  f(\bar y_i)\big({p}(x,\bar y_i)  +   {\hat h_{\lambda,\Pi}}(x,\bar y_i)\big)^+}{\sum_{j=1}^{\bar n}   \big({p}(x,\bar y_j)  +   {\hat h_{\lambda,\Pi}}(x,\bar y_j)\big)^+},\label{approxCE3}
\end{align}
for any $f\in L^2_{\Pa_Y}$. Here we estimate the $\Pa_Y$-integral in \eqref{approxCE1} by the empirical counterpart given by~\eqref{approxCE2}, for an auxiliary i.i.d.\ sample $\bar y_1,\dots,\bar y_{\bar n}$ of $\Pa_Y$. The last approximation in \eqref{approxCE3} is practically motivated such that we integrate $f$ with respect to a bona-fide conditional distribution.\footnote{The density $g_\star$ satisfies the implicit structural properties $g_\star(x,y)\ge 0$ $\Pa_X\otimes\Pa_Y $-a.s., $\int_\Ycal g_{\star}(x,y)\Pa_Y( d  y)=1$ $\Pa_X$-a.s., and $\int_\Xcal g_{\star}(x,y)\Pa_X( d  x)=1$ $\Pa_Y$-a.s.}

The estimator ${\hat h_{\lambda,\Pi}}$ in \eqref{approxCE1} requires  i.i.d.\ samples $z_{\Pa,1},\dots,z_{\Pa,n}$ of $\Pa= \Pa_X\otimes\Pa_Y$ and $z_{\Q,1},\dots,z_{\Q,n}$ of $\Q= \Pa_{(X,Y)}$. If we could only sample from the joint distribution, we would consider an i.i.d.\ sample $(x_1,y_1),\dots,(x_{3n},y_{3n})$ of $\Pa_{(X,Y)}$ of size $3n$ and set 
\begin{equation}\label{zPXY}
     z_{\Pa,i} \coloneqq  (x_{2i-1},y_{2i}) ,\quad  z_{\Q,i} \coloneqq (x_{2n+i},y_{2n+i})  ,\quad i=1,\dots,n.
\end{equation}

\begin{remark}
The sampling scheme \eqref{zPXY} may be prohibitive in terms of the required total sample size $3n$. In practice, observing an i.i.d.\ sample $(x_1,y_1),\dots,(x_{n},y_{n})$ of $\Pa_{(X,Y)}$ of size $n$, one could replace \eqref{zPXY} by $z_{\Pa,i} \coloneqq  (x_{i},y_{i+1})$, with $y_{n+1}\coloneqq y_1$, and $z_{\Q,i} \coloneqq (x_{i},y_{i})$, $i=1,\dots,n$. This comes at the cost of introducing bias though the mutual dependence of $z_{\Pa,i}$, $z_{\Pa,i+1}$, $z_{\Q,i}$. Similarly, the auxiliary i.i.d.\ sample $\bar y_1,\dots,\bar y_{\bar n}$ of $\Pa_Y$ in \eqref{approxCE2} and \eqref{approxCE3} could be obtained by setting $\bar y_i= y_i$ for some $\bar n \le n$.
\end{remark}

Under the assumptions of Theorem~\ref{thmAEbound}, the approximation \eqref{approxCE3} of \eqref{approxCE2} can be shown to be asymptotically exact as $n \to \infty$. This follows from \eqref{eqHrates} and because the $\Hcal$-norm dominates the sup-norm by assumption \eqref{asskappa}.

\subsubsection{Relevance and scope of the absolute continuity assumption}\label{subrelassu}

Assumption \eqref{assu1XY} is related to the mutual information $I(X,Y)=\int_{\Zcal} \log(g_{\star}) \,d \Pa_{(X,Y)}$ of $X$ and $Y$, the Kullback--Leibler divergence of $\Pa_{(X,Y)}$ from $\Pa_X\otimes\Pa_Y $, which is well-defined and finite if and only if \eqref{assu1XY} holds. Some literature thus refers to $g_{\star}$ as the mutual information density. 
Assumption \eqref{assu1XY} is not restrictive in practice, as the following lemma shows. 
\begin{lemma}\label{lemassu1XY}
Assume there exist measures $\mu_\Xcal$ and $\mu_\Ycal$ on $\Xcal$ and $\Ycal$ such that $\Pa_{(X,Y)}$ is absolutely continuous with respect to the product measure $\mu_\Xcal\otimes\mu_\Ycal$, with say density~$f$, 
\begin{equation}\label{eqlemass}
     d \Pa_{(X,Y)} = f  d \,(\mu_\Xcal\otimes\mu_\Ycal).
\end{equation}
Then \eqref{assu1XY} holds.
\end{lemma}
The proof is provided in Appendix~\ref{app:proof_lemma11}.

 A simple example where the assumption of Lemma \ref{lemassu1XY} holds is easily constructed. Take $\Xcal=\Ycal=\R$ and $\mu_\Xcal =\mu_\Ycal$ is the Lebesgue measure on $\R$, and consider $(X,Y)$ jointly Gaussian with non-degenerate covariance. Then \eqref{assu1XY} is satisfied. A counter example to \eqref{assu1XY} is similarly easily devised  by setting $Y=X$ for a Gaussian random variable $X$. Then the joint distribution of $(X,Y)=(X,X)$ is concentrated on the diagonal in $\R^2$ and \eqref{assu1XY} does not hold. This counter example also illustrates that exceptions to \eqref{assu1XY} are rather degenerate limiting cases.

\subsubsection{Comparison with related literature}\label{ssec_comp}
We relate our conditional distribution estimator \eqref{approxCE1} to other approaches in the literature. 

\paragraph{Kernel conditional mean embeddings}\label{secCME}
In view of the decomposition \eqref{boundEYXnew} it is natural to compare conditional expectations \eqref{eqCEfYX} calculated from a candidate density \eqref{eq:ansatz} to the conditional mean embedding introduced in \citet{songetal09}. To facilitate this comparison, we assume in this subsection that $\Hcal=\Hcal_\Xcal\otimes\Hcal_\Ycal$ is the tensor product of two separable RKHS $\Hcal_\Xcal$ and $\Hcal_\Ycal$ with measurable kernels $k_\Xcal$ and $k_\Ycal$ on $\Xcal$ and $\Ycal$, respectively. In line with \eqref{asskappa}, we also assume that $\sup_{x\in\Xcal}   k_\Xcal(x,x)<\infty$ and $\sup_{y\in\Ycal} k_\Ycal(y,y) <\infty$, so that the canonical embeddings $J_{\Pa_X}\colon\Hcal_\Xcal\to L^2_{\Pa_X}$, $J_{\Pa_Y}\colon\Hcal_\Ycal\to L^2_{\Pa_Y}$ are Hilbert--Schmidt operators. We then have $k\big((x_1,y_1),(x_2,y_2)\big) = k_\Xcal(x_1,x_2) k_\Ycal(y_1,y_2)$, and in view of \eqref{eqisoL2} we can identify the operators $J_{\Pa_X\otimes\Pa_Y}   = J_{\Pa_X}\otimes J_{\Pa_Y}$.

Combining this with \eqref{assuginL2}, we obtain kernel embeddings of the conditional distributions~\eqref{condist} in the sense that
\begin{equation}\label{eqkembcd}
  \int_\Ycal f (y) \Pa_{Y\mid  X=x}(d y) =\langle J_{\Pa_Y} f ,   g_{\star}(x,\cdot)\rangle_{L^2_{\Pa_Y}}
  =\langle  f , J_{\Pa_Y}^\ast g_{\star}(x,\cdot)\rangle_{\Hcal_Y} ,\quad \text{for $f \in \Hcal_Y$,}
\end{equation}
and similarly for $\Pa_{X\mid Y=y}(dx)$. We thus identify from \eqref{eqkembcd} the element $J_{\Pa_Y}^\ast g_{\star}(x,\cdot)\in\Hcal_Y$ as the conditional mean embedding $\mu_{Y\mid X=x}$ introduced in \citet{songetal09}. More specifically, \citet{songetal09} realize $\mu_{Y\mid X=x}$ through unbounded linear operations between the RKHS $\Hcal_X$ and $\Hcal_Y$. A rigorous theory is given by \citet{klebanovschustersullivan20}, which reveals that, albeit mathematically elegant, the linear operator approach of \citet{songetal09} comes with some practical limitations. First, it requires an elaborate analysis based on sophisticated knife-edge technical assumptions. Second, the inverse problem of recovering $\Pa_{Y\mid X=x}$ from $\mu_{Y\mid X=x}$ is ill-posed, as discussed   in \citet{sch_etal_20}, and left open as a ``fruitful avenue of research'' in \citet{klebanovschustersullivan20}. Third, the conditional mean embedding acts through \eqref{eqkembcd} only on functions $f\in\Hcal_\Ycal$.

Our approach, via the density $g_\star$, overcomes the above limitations of the traditional conditional mean embedding approach. First, our construction of $\mu_{Y\mid X=x}$ is more elementary and feasible under verifiable technical assumptions, such as \eqref{assu1} and \eqref{assuginL2}. Second, our candidate model \eqref{eq:ansatz} of $g_\star$ yields the candidate conditional distribution $\Pa_{Y\mid X=x}$ directly using identity \eqref{condist}. Third, we thus compute conditional expectations \eqref{eqCEfYX} directly for all test functions $f\in L^2_{\Pa_Y}$, and not only for $f\in \Hcal_\Ycal$.

\paragraph{Other approaches}  \citet{sug_etal_10,hin_etal_21} estimate conditional density functions with respect to Lebesgue measure. Our setup is more general. In turn, we do not obtain a direct estimator for the conditional Lebesgue density function, hence the evaluation metric in \citet{sug_etal_10,hin_etal_21} (the negative log likelihood) cannot be applied here and a direct comparison with the scores in \citet{sug_etal_10,hin_etal_21} is not possible.

Also the setup in \citet{suz_etal_09, sug_13} is restricted to continuous variables with a Lebesgue density. In fact, they assume the assumptions of Lemma~\ref{lemassu1XY} hold for the Lebesgue measures $\mu_\Xcal(dx)=dx$ and $\mu_\Ycal(dy)=dy$. In this case, the ``density $w(x,y)$'' in \citet{suz_etal_09} is equal to our density $g_\star(x,y)$. Their estimator of the density is based on a sampling scheme, similar to \citet{filipovic2024adaptive}, which is different from ours. As a consequence, they do not provide finite sample guarantees nor a central limit theorem.

There is a vast literature on (semi-)parametric distributional regression approaches, which includes structured additive models, see \citet{rue_etal_23} for a recent overview. In contrast, our approach is nonparametric functional analytic and learns the density in a reproducing kernel Hilbert space, which comes with asymptotic theory and finite sample guarantees. Isotonic distribution regression in \citet{henziziegelgneiting21} is to date specialized on univariate response variables and conditional distributions, conformal predictive distributions from  \citet{vovketal2019} are based on regression problems. \citet{shen2024engressionextrapolationlensdistributional} estimates conditional distributions only.

\section{Experiments}\label{sec:experiments}
In the following, we present a series of numerical experiments on both simulated and real data to assess the flexibility and robustness of the KDM model across a range of applications. We fix the constant prior $p=1$ in the following.


\subsection{Gaussian introductory example}\label{sec:introex}
As a first illustrative example, we consider two shifted Gaussian distributions. In particular, we let $\mathbb{P}\sim \mathcal{N}(0,1)$ and $\mathbb{Q}\sim \mathcal{N}(\mu,1)$, so that the target density takes the form
\[
g_{\star}(x)\coloneqq e^{-\frac{\mu^2}{2}+\mu x}.
\]
Our goal is twofold: first, to show that KDM is able to accurately recover $g_\star$ and second, to highlight the benefits of the Nystr\"om approximation introduced in Section~\ref{sec:lowrank} compared to the full model described in Section~\ref{sec:sampleestimator}. In line with the main theoretical result in Theorem~\ref{thmAEbound}, we seek to empirically verify that the compressed estimator in \eqref{hlambdaeqsampleLRcoor2} matches the performance of the full model in \eqref{eqRT}, while using only $m \ll n$ Nystr\"om points sampled uniformly from the $\mathbb{P}$-sample, and at the same time achieving substantial computational savings.

We set $\mu = 1/2$ and $n = 5000$, and repeat the experiment 100 times. We use a Gaussian kernel throughout this example. For simplicity, we fix the regularization parameter to $\lambda = n^{-1/2}$ and choose the kernel length scale as $\rho = \mathrm{median}\{\|z_{\mathbb{P},i} - z_{\mathbb{P},j}\|_2 : i < j\}/\sqrt{2}$.

The results are reported in Figure~\ref{fig:gausstest}. As expected, Figure~\ref{fig:gausslowrankerror} shows that, in this simple setting, the compressed KDM estimator in \eqref{hlambdaeqsampleLRcoor2} achieves performance comparable to the full model in \eqref{eqRT} using only $m\sim50$ Nystr\"om points. Furthermore, Figure~\ref{fig:gausscomputationtimes} highlights the significant reduction in computational time achieved by the compressed approach.


\begin{figure}
\begin{subfigure}[t]{0.5\textwidth}
 \includegraphics[scale=0.35]{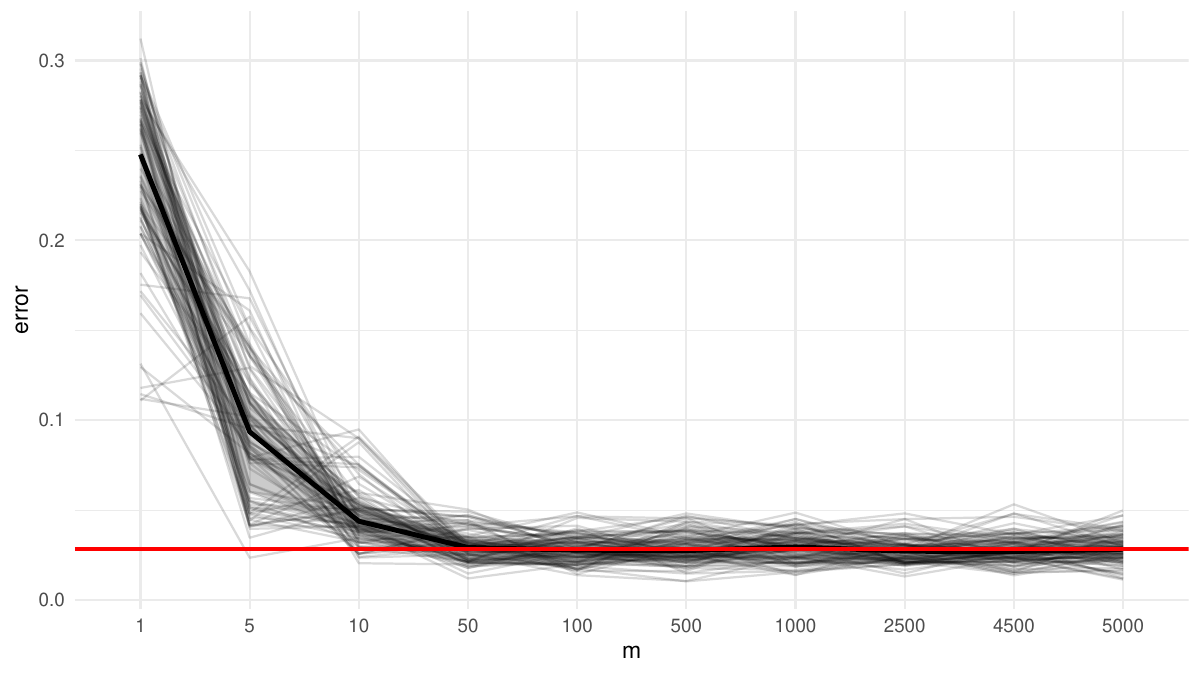}
 \caption{\label{fig:gausslowrankerror}Empirical error}
 \end{subfigure}
 \begin{subfigure}[t]{0.5\textwidth}
  \includegraphics[scale=0.35]{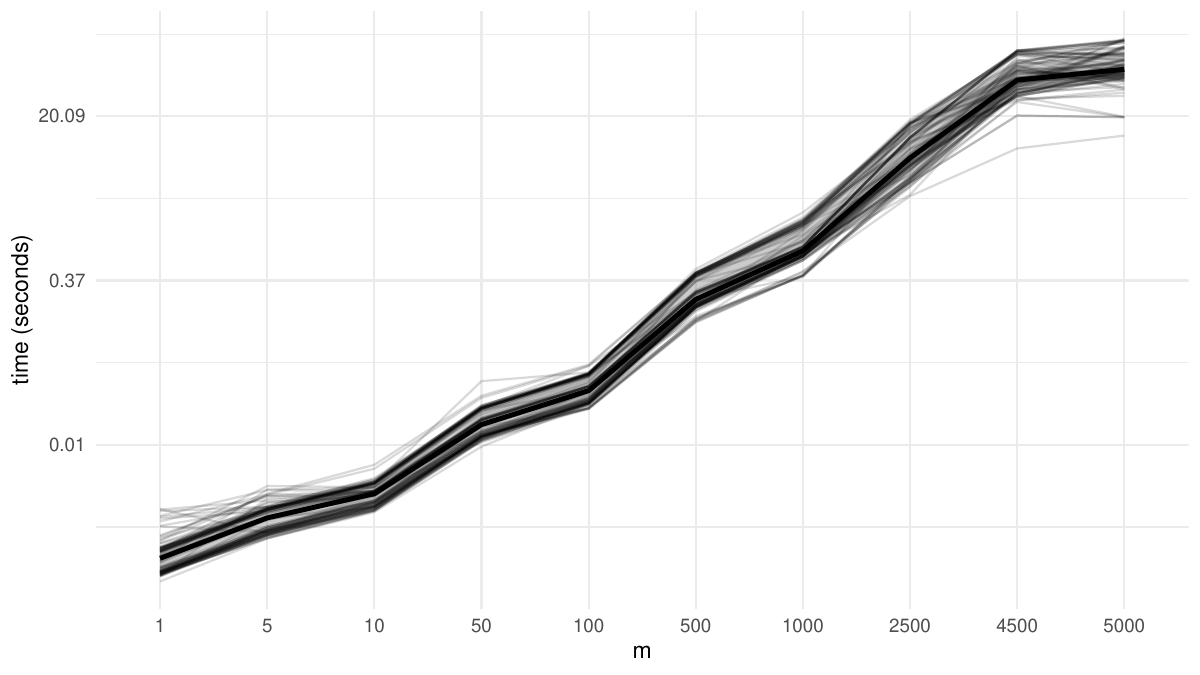}
  \caption{\label{fig:gausscomputationtimes}Runtime}
  \end{subfigure}
  \caption{\label{fig:gausstest} Left: empirical test error in \eqref{eqPML2} as a function of the number of sampled Nystr\"om points $m$ for the experiment described in Section \ref{sec:introex}. The black curve denotes the average over 100 repetitions, while the red line represents the full model in \eqref{eqRT}. Right: computational time as a function of $m$.
  }
\end{figure}


\subsection{Discrete target distribution}
As a second example, we evaluate the ability of the KDM model to recover the density in a discrete and highly non-linear setting. We consider a distribution supported on a discrete uniform $200 \times 200$ grid over $[-1.2,1.2]\times[-1.2,1.2]$. Figure \ref{fig:smiley} illustrates the density of the target distribution with respect to  the uniform distribution.

Notably, as an additional challenge, the target density is zero for most points in the support. In particular, only $3326$ out of $40000$ points (i.e., $8.3\%$) have non-zero density.

Since the Gaussian kernel used in the introductory example is not well suited to this highly non-smooth setting, we instead employ the compactly supported kernel
\begin{align}
 \quad k_{circular}(z,z')&\coloneqq \begin{cases} \frac{2}{\pi}\arccos \left (\frac{\|z-z'\|_2}{\theta} \right )- \frac{\|z-z'\|_2}{\theta}  \sqrt{1-\frac{\|z-z'\|_2^2}{\theta^2} }& \text{ if } \|z-z'\|<\theta \\
 0 & \text{otherwise}
\end{cases}\label{eq:circular}
\end{align}
introduced in \citet{genton02}. 

For $n=10^4$ and $n=10^5$, we sample points $z_{\Q,1},\ldots, z_{\Q,n}$ from the target distribution $\Q$, and $z_{\Pa,1},\ldots, z_{\Pa,n}$ from the uniform distribution $\Pa$.  For simplicity, we fix the regularization parameter to $\lambda = n^{-1/2}$ and choose the length scale as $\theta = \mathrm{median}\{\|z_{\Pa,i}-z_{\Pa,j}\|_2 : i<j\}$.

\begin{figure}
\centering

\begin{subfigure}[t]{0.4\textwidth}
  \centering
\includegraphics[width=\linewidth]{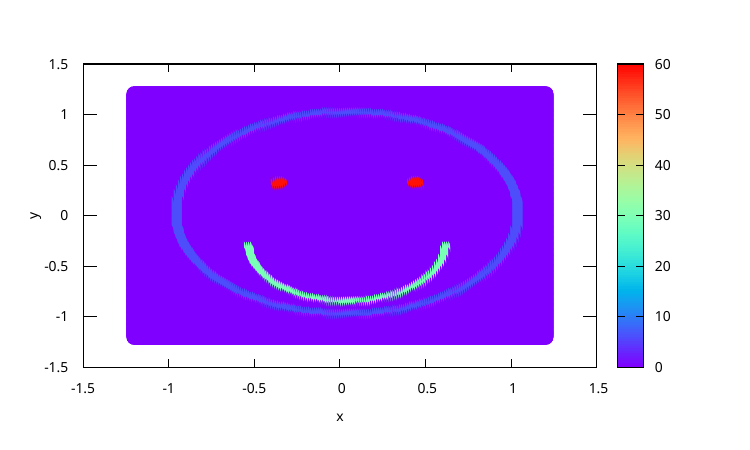}
\caption{\label{fig:smiley}True distribution}
\end{subfigure}\\
\begin{subfigure}[t]{0.4\textwidth}
  \centering
  \includegraphics[width=\linewidth]{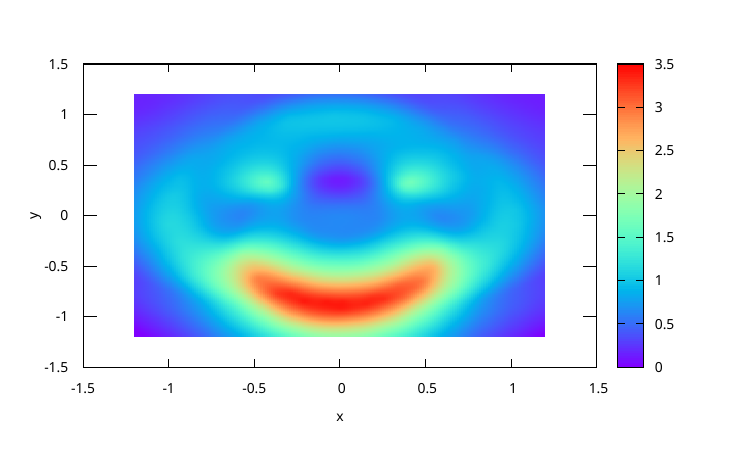}
  \caption{$n=10^4$ with circular kernel \eqref{eq:circular}}
\end{subfigure}%
\begin{subfigure}[t]{0.4\textwidth}
  \centering
  \includegraphics[width=\linewidth]{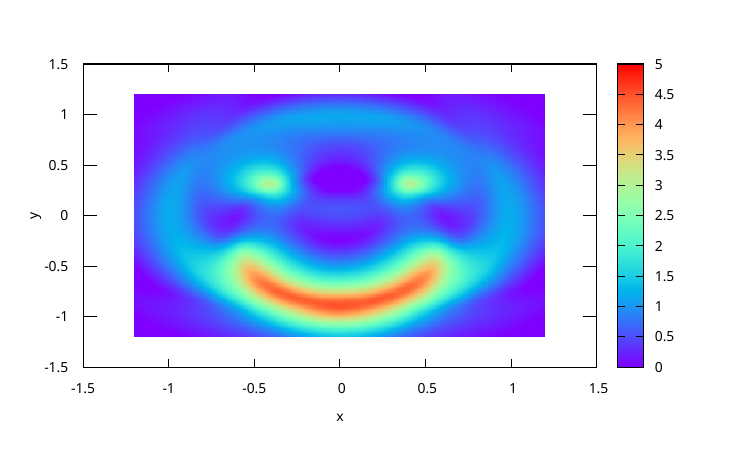}
  \caption{$n=10^5$ with circular kernel \eqref{eq:circular}}
\end{subfigure}

\caption{\label{fig:smileyest}
True and estimated \textit{smiley} distribution with KDM model. 
}
\end{figure}
Figure \ref{fig:smileyest} shows the estimated density obtained with our KDM model, compared to the ground truth. Qualitatively, the \textit{smiley} structure is already visible for $n=10^4$ and becomes more clearly defined for $n=10^5$.

\subsection{Empirical evaluation of two-sample testing}
\label{sec:emp-two-sample}
In this third example, following Section~\ref{sec:two-sample}, we use KDM to perform two-sample tests on a collection of distributions taken from \citet{zengxiatong18} and \citet{aisunzhu22}; see Appendix~\ref{sec:independencedistributions}. We expect to confirm independence for the \textit{Independent Clouds} distribution, while rejecting independence in the remaining dependent cases.
We employ a Gaussian kernel with length scale selected via the median heuristic, as in the previous sections, and set $\lambda = n^{-1/2}$.
\begin{table}
\begin{center}
\begin{footnotesize}
\begin{tabular}{c | cc | cc}
\hline
 &  \multicolumn{2}{c}{n=1000} & \multicolumn{2}{c}{n=5000} \\
model & HSIC & KDM & HSIC & KDM \\
\hline
IndependentClouds &  0.04 & 0.06 & 0.06 & 0.05 \\
W                 &  1.00 & 0.74 & 1.00 & 1.00 \\
Diamond           &  1.00 & 0.97 & 1.00 & 1.00 \\
Parabola          &  1.00 & 0.57 & 1.00 & 1.00 \\
TwoParabola       &  1.00 & 0.74 & 1.00 & 1.00 \\
Circle            &  1.00 & 1.00 & 1.00 & 1.00 \\
Variance          &  1.00 & 0.78 & 1.00 & 1.00 \\
Log               &  1.00 & 0.94 & 1.00 & 1.00 \\
\hline
\end{tabular}
\end{footnotesize}
\end{center}
 \caption{\label{tab:indep}Independence testing.
Rejection rates at the 5\% significance level for samples from \textit{IndependentClouds} (first row) and for dependent samples from various distributions taken from \citet{zengxiatong18}. The column ``HSIC'' reports the results of the test proposed in \citet{pfisteretal17}.}
 \end{table}
Table \ref{tab:indep} reports the rejection rates of the test statistic \eqref{eq:teststat} at the 5\% significance level, estimated over $1000$ datasets, for training sample sizes $n=1000$ and $n=5000$. The test relies on the Gamma approximation \ref{lthmnullh4} in Theorem~\ref{lthmnullh}. As expected, the performance of KDM improves with larger sample sizes, in line with the asymptotic validity of \eqref{eq:teststat}. For $n=5000$, KDM achieves performance comparable to the specialized HSIC test of \citet{pfisteretal17}. 

\subsection{Conditional distribution estimation}
In the next two examples, we follow Section~\ref{sec:cond-distr-est} and we use KDM to estimate conditional distributions. First, we consider simulated samples generated from mixture distributions. Second, we estimate the conditional distribution of stock returns using real market data. For these tasks, we use a Gaussian kernel, with length scale and regularization parameter selected via cross validation.

\subsubsection{Simulation study from mixture models}\label{sec:simmixture}
We assess the conditional distributions and conditional expectations \eqref{approxCE3} through scoring rules as proposed by \citet{gneitingraftery07}, and confront KDM with the locally smoothed kernel density estimator \citep{rac_08}. 
We generate a mixture of $j \in \{1,2,3\}$ Gaussian components on $\mathbb{R}^2 \times \mathbb{R}^2$, with random correlation matrices \citep{ilyahensen21} and mean vectors drawn from $\Ucal(-0.2,0.2)$. Mixture weights are sampled from the probability simplex, and $3n$ i.i.d.\ samples $(x_i,y_i)$ are drawn from the resulting distribution.


To generate the samples $\bm z_\Pa$ and $\bm z_\Q$, we follow the construction in \eqref{zPXY}. Given these samples, KDM is trained with the Gaussian kernel and 20-fold cross-validation. 
The experiment is repeated 200 times.

As a benchmark, we estimate a nonparametric locally smoothed kernel density following \citet{liracine06}, using all $3n$ data points and the \texttt{np} package in \textsf{R}. The benchmark metric is based on the energy scoring rule, as implemented in the \textsf{R} package \href{https://github.com/FK83/scoringRules}{scoringRules}, for out-of-sample data $  x_1,\ldots,   x_L \in \R^2$ and $  y\in \R^2$ generated from the mixture distributions described above, 
\[
 ES_{\Mcal}(  y)\coloneqq \frac{1}{L}\sum _{l=1}^mw^{\Mcal}(  x_l,  y)\|  y-  x_l\|_2-\frac{1}{2L^2}\sum _{l,m=1}^mw^{\Mcal}(  x_l,  y)w^{\Mcal}(  x_m,  y)\|  x_l-  x_m\|_2,
\]
where the weights $w^{\Mcal}(  x,   y)$ are taken to be the conditional densities of $  y|  x$ evaluated for $\Mcal \in \{\text{KDM, \texttt{np}}\}$. We then consider the average
$
 \text{energy score differential }  \frac{1}{L}\sum _{i=1}^L (ES_{\texttt{np}}(  y_i)-ES_{\text{KDM}}(  y_i)),
$
 for $  y_1,\ldots,   y_L$ from the joint sample $(  x_1,  y_1),\ldots , (  x_L,  y_L)$.

Figure \ref{fig:energy} shows the distribution of the energy score differential over the 200 simulation runs, and for one, two, and three clusters in the Gaussian mixture for $n=1000$.
The mean energy score differential can be seen to be in favor of KDM. The empirical distributions are also pronouncedly skewed in favor of KDM. Note that the computational effort for bandwidth selection required by \texttt{np}, in particular in larger data sets and higher dimensions is substantial and by far exceeds KDM's requirements. We therefore refrain from considering such larger data sets.

\begin{figure}
\begin{center}
  \includegraphics[scale=0.75]{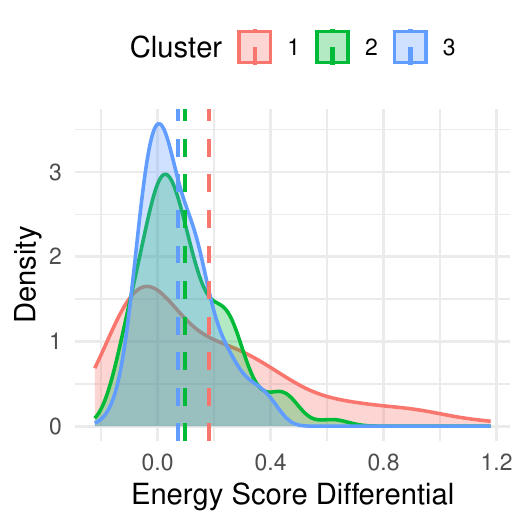}
  \end{center}
  \caption{\label{fig:energy}Energy score differential. The figure shows the distribution of the difference between the energy score computed from KDM and the nonparametric kernel density estimation as described in \citet{liracine06}. 
  Higher is better.}
\end{figure}
\subsubsection{Conditional distribution of stock returns}\label{sec:stockpred}
To showcase KDM with real and higher-dimensional data, we consider the joint distribution of eleven realized monthly stock portfolio returns $  Y_{t+1}$, of which five are from  \citet{FAMA20151}, four from \citet{houxuezhang14},  one from \citet{hekellymanela17}, and one is a momentum factor return, along  with eleven predictive variables $  X_t$: book-to-market  (BM),  net-equity-expansion (ntis), inflation growth  (infl), stock variance (svar), dividend yield (DP), default yield (DFY), term spread (TMS),  civilian unemployment rate (UNRATE), consumption growth (CONSGR), and the Chicago Fed National Activity Index (CFNAI)  from Jan 1963 to Dec 2022 (720 months). Both, the returns, and the conditioning covariates, are indexed by time. From these data, we have $\Zcal=\R^{d}\times \R^{d}$,  with  $d=11$, and thus total dimension adding up to 22. This is too high-dimensional for the locally smoothed nonparametric kernel density estimation from \citet{liracine06}, and we use instead a Gaussian distribution on $\Zcal$ as a benchmark, whose moments are estimated from sample averages, from which we compute conditional moments. We use monthly expanding training windows starting in Jan 1963, with lengths ranging from 200 to 719 months, each followed by one test month (the first test month is Sep 1979, the last is Dec 2022). For KDM, kernel parameters are selected via $8$-fold cross-validation. 

We furthermore use the statistical scoring rule \(\mathcal{S}: \mathbb{R}^d \times \mathbb{R}^d \times \mathbb{S}^{d}_{++} \to \mathbb{R}\) proposed by \citet{dawidsebastiani99},
\begin{equation*} 
\mathcal{S}( {x},  {\mu}, \bm{\Sigma}) \coloneqq \log \det \bm{\Sigma} + ( {x} -  {\mu})^{\top} \bm{\Sigma}^{-1} ( {x} -  {\mu}),
\end{equation*}
in particular the score differential
\begin{equation}\label{eqRcaltT}
 \Scal_{t,T,\text{OOS}} \coloneqq \frac{1}{T - t} \sum_{s=t}^{T-1} \big(\Scal(  Y_{s+1},   \mu _{  Y_{s+1}|  X _s}^{\text{Gauss}}, \bm \Sigma _{  Y_{s+1}|  X _s}^{\text{Gauss}}) - \Scal(  Y_{s+1},  \mu _{  Y_{s+1}|  X _s}^{\text{KDM}}, \bm \Sigma _{  Y_{s+1}|  X _s}^{\text{KDM}})\big),
\end{equation}
where $\bm \Sigma _{  Y_{s+1}|  X _s}^{\text{KDM}}$ and $\bm \Sigma _{  Y_{s+1}|  X _s}^{\text{Gauss}}$ are conditional covariance matrices computed from KDM and a conditional Gaussian distribution estimated from sample averages.


Figure \ref{fig:stocks2} plots the logarithm of the expanding out-of-sample score differential $\Scal_{t,T,\text{OOS}}$ over time $T$, which remains uniformly positive. This indicates that KDM yields a consistent improvement in predictive distributions for financial data relative to standard benchmarks, such as the conditional Gaussian model. Overall, KDM demonstrates strong empirical performance on both simulated and real data.



\begin{figure} 
\begin{center}
 \includegraphics[scale=0.42]{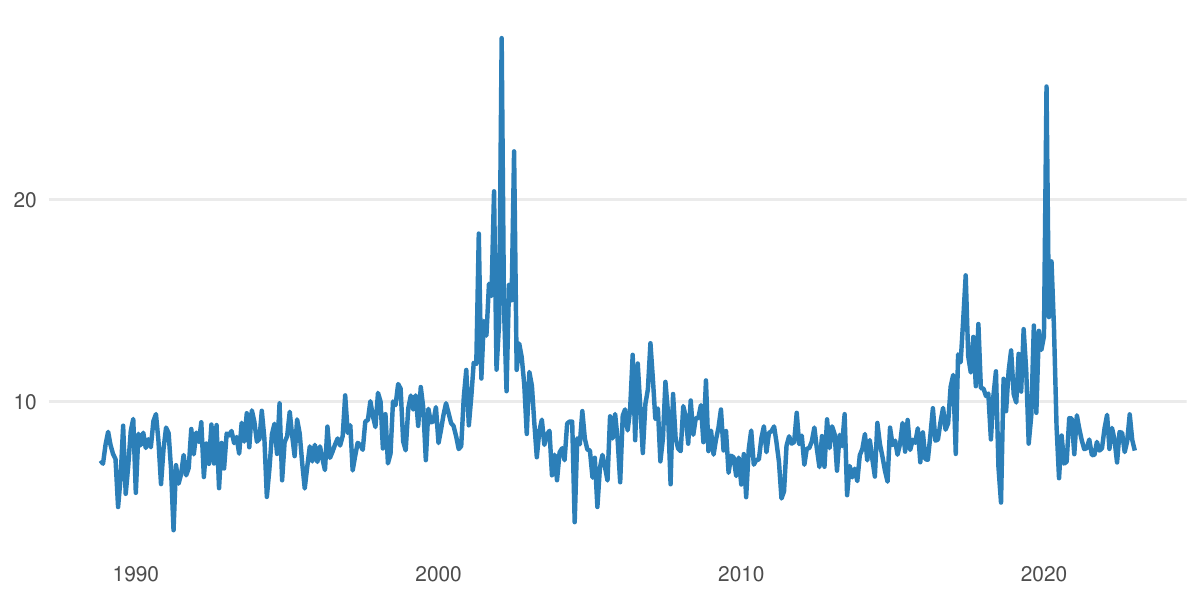} 
\end{center}
  \caption{\label{fig:stocks2}This figure shows the logarithm of the score differential \eqref{eqRcaltT}. Higher is better. The data are eleven monthly US stock returns, and eleven conditioning variables from 1963 up to 2022.}
\end{figure}


\section{Conclusion}\label{sec:conclusion}
Kernel density machines (KDM) is a comprehensive data-driven framework for estimating the Radon--Nikodym derivative (density) $\frac{d\mathbb Q}{d\mathbb P}$ from i.i.d.\ samples of $\mathbb P$ and $\mathbb Q$ in a reproducing kernel Hilbert space. KDM is computable in particular with large data sets, and allows both finite-sample and asymptotic inference. Accordingly, we illustrated its use within hypothesis testing and conditional distribution estimation. Given the reliance of learning problems on the data-generating probability law and the law induced by a model, KDM can be used in many scenarios ranging from transfer learning problems, and generative models to causal inference.


\begin{singlespacing}
\bibliographystyle{ecta}
\bibliography{Bibliography}
\end{singlespacing}


\appendix

\section{Proofs}\label{appproofs}
This appendix collects all proofs.

\subsection{Proof of Lemma~\ref{lemJast0Jast}}
\label{app:proof_lemma2}
For any such $f$ we have $f\trueg\in L^2_{\Pa}$, $f\in L^2_{\Q}$, and
\[   J_\Pa^\ast (f\trueg)=\int_{\Zcal}
   k(\cdot,z) f(z)\trueg(z)\, \Pa(d{z})=\int_{\Zcal}
  k(\cdot,z)f(z)\, \Q(d{z})=J_\Q^\ast f  ,\]
   as desired.

\subsection{Proof of Lemma \ref{lemregerror0}}
\label{app:proof_lemma3}
\ref{lemregerror00}: This follows from  the definition of the projection.

\ref{lemregerror01}: In view of Lemma~\ref{lemJast0Jast}, and by the projection property, we can rewrite \eqref{hlambdaeqpre} as
\begin{equation}\label{hlambdaeqgstar}
 h_\lambda = ( J_\Pa^\ast J_\Pa + \lambda)^{-1} J_\Pa^\ast  (  g_\star -  {p} )=( J_\Pa^\ast J_\Pa + \lambda)^{-1} J_\Pa^\ast ( g_0 - {p} ).
\end{equation}
Convergence now follows as in the proof of  \citet[Lemma 2.3]{bou_fil_22}.

\ref{lemregerror02}: This statement is elementary. 

\ref{lemregerror03}: By assumption we can rewrite \eqref{hlambdaeqgstar} as $h_\lambda =  ( J_\Pa^\ast J_\Pa + \lambda)^{-1} (J_\Pa^\ast J_\Pa)^{1+{{\nu}}} \eta_{{\nu}}$, for  $\eta_{{\nu}}\coloneq (J_\Pa^\ast J_\Pa)^{-{\nu}}h_0$. Using the notation and setup as in \citet[Appendix B.5 and B.6]{bou_fil_22}, we can expand $h_\lambda$ as 
\[h_\lambda = \sum_{i\in I} \frac{\mu_i^{1+{{\nu}}}}{\mu_i+\lambda} \langle \eta_{{\nu}},u_i\rangle_\Hcal u_i,\]  
and $(J_\Pa^\ast J_\Pa)^{-{{\nu}}}h_\lambda = \sum_{i\in I} \frac{\mu_i}{\mu_i+\lambda} \langle \eta_{{\nu}},u_i\rangle_\Hcal u_i$, and $J_\Pa h_\lambda =   \sum_{i\in I}  \mu_i^{1/2}\frac{\mu_i^{1+{{\nu}}}}{\mu_i+\lambda} \langle \eta_{{\nu}},u_i\rangle_\Hcal v_i$, where $\{u_i\}_{i\in I}$ is an orthonormal system in $\Hcal$ of eigenfunctions of $J_\Pa^\ast J_\Pa$ with eigenvalues $\mu_i>0$, for a countable index set $I$. The functions $v_i=\mu_i^{-1/2} J_\Pa u_i$, in turn, form an orthonormal system in $L^2_{\Pa}$ of eigenfunctions of $J_\Pa J_\Pa^\ast$ with the same eigenvalues $\mu_i$. Properties \eqref{hb1s} and \eqref{hb2s} now  follow from dominated convergence. The convergence rates \eqref{hb3s} and \eqref{hb4s} follow from
\begin{align*}
\| h_\lambda-h_0\|_\Hcal^2 & = \sum_{i\in I}   \Big(\frac{\lambda \mu_i^{{\nu}}}{\mu_i+\lambda}\Big)^2 \langle \eta_{{\nu}},u_i\rangle_\Hcal^2  =\lambda^{2{{\nu}}} \sum_{i\in I}   \Big(\frac{\lambda^{1-{{\nu}}} \mu_i^{{\nu}}}{\mu_i+\lambda}\Big)^2 \langle \eta_{{\nu}},u_i\rangle_\Hcal^2\\
& \le \lambda^{2{{\nu}}} (1-{{\nu}})^2 \|\eta_{{\nu}}\|_\Hcal^2
\end{align*}
and 
\begin{align*} 
\| J_\Pa h_0 -J_\Pa h_\lambda \|_{L^2_{\Pa}}^2 &=  \sum_{i\in I}   \Big(\frac{\lambda \mu_i^{1/2+{{\nu}}}}{\mu_i+\lambda}\Big)^2 \langle \eta_{{\nu}},u_i\rangle_\Hcal^2  =\lambda^{1+2{{\nu}}} \sum_{i\in I}   \Big(\frac{\lambda^{1/2-{{\nu}}} \mu_i^{1/2+{{\nu}}}}{\mu_i+\lambda}\Big)^2 \langle \eta_{{\nu}},u_i\rangle_\Hcal^2 \\
&\le \lambda^{1+2{{\nu}}} (1/2+{{\nu}})^2 \|\eta_{{\nu}}\|_\Hcal^2 ,
\end{align*}
where we used Young's inequality $\lambda^{1-t}\mu_i^t \le (1-t)\lambda + t \mu_i$ for $t\in [0,1]$.

\subsection{Proof of Theorem~\ref{thmAC}}
\label{app:proof_thm4}
The proof of Theorem~\ref{thmAC} relies on the following key lemma. 
\begin{lemma}\label{lemAC} 
\begin{enumerate}
 
  \item\label{lemAC1} The $\Hcal$-valued sample variable defined by 
\begin{equation}\label{Deltadef}
    {u_\lambda}\coloneqq n^{-1}(S_\Q^\ast \bm 1    -  S_\Pa^\ast ( \bm {p} + S_\Pa h_\lambda)) - (  J_\Q^\ast 1     -    J_\Pa^\ast ( {p} +   J_\Pa h_\lambda))
\end{equation} 
has mean zero and ${u_\lambda}\to 0$ in $\Hcal$ a.s.\ as $n\to\infty$.

\item\label{lemAC2} $n^{1/2}{u_\lambda} \to \Ncal(0,Q_\lambda)$ in distribution as $n\to\infty$, where the covariance operator $Q_\lambda$ is given in \eqref{Clamdef}.

\item\label{lemAC4} $(n^{-1} S_\Pa^\ast S_\Pa+\lambda)^{-1} \to (J_\Pa^\ast J_\Pa+\lambda)^{-1}$ in operator norm a.s.\ as $n\to\infty$.
\end{enumerate}
\end{lemma}

\begin{proof}
\ref{lemAC1}: We can write ${u_\lambda}=n^{-1}\sum_{i=1}^n \xi_i$, where
 \begin{equation} \label{eqxii}
    \xi_i \coloneqq  k(\cdot, z_{\Q,i})  - ( {p}(z_{\Pa,i})+ h_\lambda(z_{\Pa,i})) k(\cdot,z_{\Pa,i})  - (  J_\Q^\ast 1     -    J_\Pa^\ast ({p} +   J_\Pa h_\lambda))
\end{equation}
are i.i.d.\ $\Hcal$-valued random variables with zero mean under the sampling measure, say, $\bm P\coloneqq(\Pa\otimes\Q)^{\otimes \infty}$. 
In view of \eqref{asskappa}, we obtain the bounds on the operator norms
\begin{equation}\label{Jopbounds}
    \|J_\Pa\|,\, \|J_\Pa^\ast\|,\, \|J_\Q \|,\, \|J_\Q^\ast\|\le  \kappa .
\end{equation}
Using \eqref{asspi}, we obtain that the random variables
\begin{equation}\label{xiCGb}
 \begin{aligned}
    \|\xi_i\|_\Hcal &\le \|k(\cdot, z_{\Q,i})\|_\Hcal  + | {p}(z_{\Pa,i})| \|k(\cdot, z_{\Pa,i})\|_\Hcal + |\langle h_\lambda,k(\cdot,z_{\Pa,i})\rangle_\Hcal| \| k(\cdot,z_{\Pa,i})\|_\Hcal \\
    &\quad + \| J_\Q^\ast\|      +   \| J_\Pa^\ast\| \|p\|_\infty +  \| J_\Pa^\ast\| \|J_\Pa\| \|h_\lambda\|_\Hcal\\
    &\le 2(  \kappa +\|p\|_\infty  \kappa +\|h_\lambda\|_\Hcal  \kappa^2) \eqqcolon c_\xi
\end{aligned}
\end{equation}
are uniformly bounded. The claim follows from the law of large numbers, $n^{-1}\sum_{i=1}^n \xi_i\to 0$ a.s., see \citet[Theorem 2.1]{hof_pis_76}.

\ref{lemAC2}: Using the above, the functional CLT applies such that $n^{-1/2}\sum_{i=1}^n \xi_i \to \Ncal(0,Q_\lambda)$ in distribution, see \citet[Theorem 3.6]{hof_pis_76}. The covariance operator $Q_\lambda$ is given by its action on test functions $f,g\in\Hcal$,
\begin{align*}
    \langle Q_\lambda f, g\rangle_\Hcal &= \E_{\bm P}[ \langle \xi_i,f\rangle_\Hcal \langle \xi_i,g\rangle_\Hcal] = \E_{\bm P}[ \langle \xi_{\Q,i},f\rangle_\Hcal \langle \xi_{\Q,i},g\rangle_\Hcal] + \E_{\bm P}[ \langle \xi_{\Pa,i},f\rangle_\Hcal \langle \xi_{\Pa,i},g\rangle_\Hcal]\\
    &= \langle J_\Q f, J_\Q g\rangle_{L^2_{\Q}} - \langle  f, J_\Q^\ast 1\rangle_\Hcal \langle  g, J_\Q^\ast 1\rangle_\Hcal+ \big\langle ({p} + J_\Pa  h_\lambda) J_\Pa   f , ({p} + J_\Pa  h_\lambda)J_\Pa   g  \big\rangle_{L^2_{\Pa }}   \\
    &\quad -   \langle f ,J_\Pa  ^\ast ({p} + J_\Pa  h_\lambda) \rangle_\Hcal \langle g, J_\Pa  ^\ast ({p} + J_\Pa  h_\lambda) \rangle_\Hcal\\
    &= \langle J_\Q^\ast J_\Q f, g\rangle_\Hcal - \big\langle \langle f ,J_\Q^\ast 1\rangle_\Hcal J_\Q^\ast 1, g\big\rangle_\Hcal \\
    &\quad +\big\langle J_\Pa ^\ast  \diag({p} + J_\Pa  h_\lambda)^2  J_\Pa   f  , g\rangle_\Hcal - \big\langle \langle f ,J_\Pa  ^\ast ({p} + J_\Pa  h_\lambda) \rangle_\Hcal J_\Pa  ^\ast ({p} + J_\Pa  h_\lambda), g\big\rangle_\Hcal .
\end{align*}
where we decompose $\xi_i=\xi_{\Q,i}-\xi_{\Pa,i}$ into the components $\xi_{\Q,i}\coloneqq k(\cdot, z_{\Q,i})-J_\Q^\ast 1$ and $\xi_{\Pa,i}\coloneqq( {p}(z_{\Pa,i})+ h_\lambda(z_{\Pa,i}))k(\cdot,z_{\Pa,i})-J_\Pa^\ast ({p} +   J_\Pa h_\lambda)$, which have mean zero and are independent under the sampling measure $\bm P$. This proves \eqref{Clamdef}.

{\ref{lemAC4}}: This follows as in \citet[Lemma B.2]{bou_fil_22}.
\end{proof}

We can now prove Theorem~\ref{thmAC}. We define $  {C_\lambda} =  J_\Pa^\ast J_\Pa + \lambda $ and $ b =  J_\Q^\ast 1-J_\Pa^\ast {p}  $ and the sample analogues $\hat C_\lambda =n^{-1} S_\Pa^\ast S_\Pa + \lambda $ and $\hat b =n^{-1}(S_\Q^\ast \bm 1-S_\Pa^\ast \bm {p}) $. Using \eqref{hlambdaeqsample} and \eqref{hlambdaeqNS}, we then decompose
\begin{equation}\label{hhhatdecX}
\begin{aligned}
    \hat h_\lambda - h_\lambda &=  \hat C_\lambda^{-1} \hat b - C_\lambda^{-1} b  = \hat C_\lambda^{-1}(\hat b - b) - (C_\lambda^{-1}-\hat C_\lambda^{-1}) b \\
    & = \hat C_\lambda^{-1}(\hat b - b) -\hat C_\lambda^{-1}(\hat C_\lambda - C_\lambda) C_\lambda^{-1}b = \hat C_\lambda^{-1} \big( \hat b - b - (\hat C_\lambda - C_\lambda) h_\lambda\big)=\hat C_\lambda^{-1} {u_\lambda},
    \end{aligned}
\end{equation}
where ${u_\lambda}$ is defined in \eqref{Deltadef}. 
Part~\ref{thmAC1} now follows from Lemma~\ref{lemAC}\ref{lemAC1} and \ref{lemAC4}. Part~\ref{thmAC2} follows from Lemma~\ref{lemAC}\ref{lemAC2} and \ref{lemAC4} and Slutsky's theorem. This completes the proof of Theorem~\ref{thmAC}.

We also provide the following auxiliary finite-sample guarantees, which hold without assuming well-posedness \eqref{asswellposed} nor source condition~\eqref{sourcecond}, and will be used in the proof of Theorem~\ref{thmAEbound}.\footnote{In fact, combining Proposition~\ref{propFSGaux} with source condition~\eqref{sourcecond} and Lemma~\ref{lemregerror0} gives the same error rates for $\hat h_\lambda$ as for the low-rank approximation $\hat h_{\lambda,\Pi}$ in Theorem~\ref{thmAEbound} but without assuming any lower bound on $n$ as in \eqref{lambdancondn}.}  Comparing \eqref{FSGeqH} and \eqref{FSGeqL2} shows that embedding the sample difference $\hat h_\lambda - h_\lambda$ into $L^2_\Pa$ improves the error rate by a factor of $\lambda^{1/2}$.

\begin{proposition}\label{propFSGaux}
Finite-sample guarantees: let $\eta\in (0,1)$.
\begin{enumerate}
 \item\label{propFSGaux1} With sampling probability of at least $1-\eta$, it holds
\begin{equation}
        \| \hat h_\lambda - h_\lambda\|_\Hcal \le  C_{FS}(\eta,\|h_\lambda\|_\Hcal)\,  \lambda^{-1} n^{-1/2},\label{FSGeqH}
      \end{equation}
 for the coefficient 
\begin{equation}
    C_{FS}(\eta, s)\coloneqq 2\kappa\sqrt{2\log(2/\eta)}(  1+\|p\|_\infty +s  \kappa ). \label{CSdef}
\end{equation}

 \item\label{propFSGaux2} With sampling probability of at least $1-2\eta$, the bounds  \eqref{FSGeqH} and 
\begin{equation}
     \| { J_\Pa}\hat h_\lambda - { J_\Pa} h_\lambda\|_{{ L^2_\Pa}} \le  C_{FS}({\eta},\|h_\lambda\|_\Hcal)\,  { C_J}({\eta},\lambda n)\,\lambda^{{ -1/2}} n^{-1/2},\label{FSGeqL2}
      \end{equation}
  hold jointly, for the coefficient 
\begin{equation}
    C_J(\eta,s)\coloneq (1+\kappa^2 s^{-1})\eta^{-1/2} .\label{CJdef}
\end{equation} 
\end{enumerate}
\end{proposition}

\begin{proof} 
First, consider the $\Hcal$-valued sample variable $u_\lambda$ given in \eqref{Deltadef}. Using the bounds derived in the proof of Lemma~\ref{lemAC}, Hoeffding's inequality \citep{pin_94} 
applies, which bounds the tail probabilities, $\bm P \big[\|{u_\lambda}\|_\Hcal > \tau\big]\le   2 \e^{-\frac{\tau^2   n}{ 2 c_\xi^2}}$ for any $\tau\ge 0$, for $c_\xi$ given in \eqref{xiCGb}. This proves that
\begin{equation}\label{FSulambda}
     \text{$\|{u_\lambda}\|_\Hcal \le  C_{FS}(\eta,\|h_\lambda\|_\Hcal)  n^{-1/2}$, with probability $1-\eta$.}
 \end{equation}

Then, considering the decomposition~\eqref{hhhatdecX} as in the proof of Theorem~\ref{thmAC}, part~\ref{propFSGaux1} of the proposition follows from \eqref{FSulambda} and the bound on the operator norm $\|\hat C_\lambda^{-1}\|\le\lambda^{-1}$. 
    
For part~\ref{propFSGaux2}, we use \cite[Lemma 7.1]{bac_24}, which is based on \cite[Lemma 5]{mou_ros_22}, which shows that
\begin{equation}\label{bdJPahatA}
\E_{\bm P}[ \|J_\Pa \hat C_\lambda^{-1}  \|^2]\le \lambda^{-1}( 1 + \kappa^2 (\lambda n)^{-1})^2.
\end{equation}
Indeed, for any $h\in\Hcal$, we have $\|J_\Pa \hat C_\lambda^{-1} h \|_{L_\Pa^2}^2 =    \langle   h, \hat C_\lambda^{-1} J_\Pa^\ast J_\Pa \hat C_\lambda^{-1} h\rangle_{L^2_\Pa}$, which is exactly the expression in \cite[LHS of (7.25)]{bac_24}. \\
The claim follows as $\|J_\Pa \hat C_\lambda^{-1}  \|^2 =\sup_{\|h\|_\Hcal\le 1} \|J_\Pa \hat C_\lambda^{-1} h \|_{L_\Pa^2}^2 $, and \cite[RHS of (7.25)]{bac_24} is bounded by the RHS of \eqref{bdJPahatA} times $\|h\|_\Hcal$, which proves~\eqref{bdJPahatA}. From \eqref{bdJPahatA}, using Markov's inequality, it follows that with sample probability of at least $1-\eta$, we have
\begin{align*}
\|J_\Pa \hat C_\lambda^{-1}  \|\le C_J(\eta,\lambda n) \lambda^{-1/2} ,
\end{align*}
for the coefficient $C_J(\eta,s)$ defined in \eqref{CJdef}. Combining this with \eqref{FSulambda}, and using the elementary fact that $\bm P[\cap_i E_i] \ge 1-\sum_i \bm P[E_i^c]$ for any events $E_i$, we obtain that \eqref{FSGeqH} and \eqref{FSGeqL2} hold jointly with sample probability $1-2\eta$.
\end{proof}

\subsection{Proof of Lemma~\ref{lemRT}}
\label{app:proof_lemma5}
 Define the positive operator $A\coloneq S_\Pa^\ast S_\Pa + n\lambda$, so that we can write \eqref{hlambdaeqsample} as
 \begin{equation}\label{eqproofhath}
     \hat h_\lambda = A^{-1} S^\ast \begin{bmatrix}
-\bm p\\
\bm 1
\end{bmatrix}.
 \end{equation}  
Then $S A S^\ast$ is positive semidefinite and $\ker (SAS^\ast)=\ker S^\ast$. Indeed, let $u\in\ker (SAS^\ast)$. Then $0=u^\top (SAS^\ast)u  = \langle A S^\ast u, S^\ast u\rangle_\Hcal$, and therefore $S^\ast u=0$. Hence $\Ima(SAS^\ast)=\Ima S$, and as $\bm B \bm B^+$ is the orthogonal projection onto $\Ima \bm B$ for any matrix $\bm B$, we obtain
 \[ (SAS^\ast) (SAS^\ast)^+ S S^\ast = S S^\ast.\]
 As $\Ima (A S^\ast)\subseteq \Ima S^\ast$, this identity implies $AS^\ast  (SAS^\ast)^+ S S^\ast =   S^\ast$. Left-multiplying the latter equation by $A^{-1}$ gives $A^{-1}S^\ast = S^\ast  (SAS^\ast)^+ S S^\ast$, which in view of \eqref{eqproofhath} yields \eqref{eqRT}. This completes the proof of Lemma~\ref{lemRT}.

\subsection{Proof of Lemma \ref{lemapplocal}}
 \label{app:proof_lemma6}
We define 
\begin{equation}\label{Vdef}
    V\coloneq k(\cdot,\bm z_\Pi^\top)\bm R : \R^\ell\to \Hcal
\end{equation}
with adjoint given by $V^\ast h=\bm R^\top \langle k(\bm z_\Pi,\cdot),h\rangle_\Hcal$. In view of \eqref{R1}, it satisfies $V V^\ast =P_\Pi$, and 
\begin{equation}\label{VastVeq}
    V^\ast V = \bm R^\top \bm K_{\Pi,\Pi}\bm R=\bm R^\top (\bm R\bm R^\top)^+\bm R=\bm R^+\bm R=\bm I_\ell
\end{equation}
where we used that $\rank\bm R=\ell$. Therefore, the operator identity holds
\begin{equation}\label{Bdefeq}
    A\coloneq (  P_\Pi  S_\Pa^\ast S_\Pa  P_\Pi  + n\lambda)^{-1}    P_\Pi =  {V} ( {V^\ast}     S_\Pa^\ast S_\Pa  {V} +  n\lambda)^{-1} {V^\ast}  .
\end{equation} 
By \eqref{hlambdaeqsampleLR}, we have
${\hat h_{\lambda,\Pi}} = A (S^\ast_\Q \bm 1 - S_\Pa^\ast \bm p)$.
Using the identities $S_\M V = \bm L_\M$ and $V^\ast S_\M^\ast = \bm L_\M^\top$ then yields \eqref{hlambdaeqsampleLRcoor1}.

For the second equality, we define $\bm B\coloneq \bm K_{\Pi,\Pa}\bm K_{\Pa,\Pi} + \lambda n \bm K_{\Pi,\Pi}$, and use the fact that $\Ima\bm K_{\Pi,\Pi}=\Ima\bm K_{\Pi,:}$, as $\bm K$ is symmetric positive semidefinite. It follows that $\Ima \bm B = \Ima(\bm K_{\Pi,\Pa}\bm K_{\Pa,\Pi}) + \Ima\bm  K_{\Pi,\Pi}  =\Ima\bm  K_{\Pi,\Pi}=\Ima(\bm R\bm R^\top)=\Ima\bm R$. Hence there exists a $u\in\R^m$ such that $\bm R u = 
\bm K_{\Pi,\Q}\bm 1 -\bm K_{\Pi,\Pa}\bm p $. Using this and the above identities, we can re-express the right hand side of \eqref{hlambdaeqsampleLRcoor1} as
\[ k(\cdot,\bm z_\Pi^\top)\bm R
\big(\bm R^\top \bm B\bm R\big)^{-1}\bm R^\top \bm R u.\]
We claim that 
\begin{equation}\label{eqproofR}
   \bm R
\big(\bm R^\top \bm B\bm R\big)^{-1}\bm R^\top \bm R = \bm B^+ \bm R. 
\end{equation}
Indeed, as the image of both sides of \eqref{eqproofR} lies in $\Ima\bm R=\Ima \bm B$, we can left-multiply both sides by $\bm R^\top\bm B$, which leads to the equivalent trivial identity $\bm R^\top\bm R= \bm R^\top\bm R$. Identity \eqref{eqproofR} yields \eqref{hlambdaeqsampleLRcoor2}, which completes the proof of Lemma~\ref{lemapplocal}.

\subsection{Proof of Theorem \ref{thmAEbound}}
\label{app:proof_thm}
The proof is based on selected arguments in the proof of \cite[Theorem 2]{rud_cam_ros_15}. Mimicking their notation, we denote $C\coloneq J_\Pa^\ast  J_\Pa$ and $C_\lambda\coloneq C+\lambda I_\Hcal$ and their sample analogues $\hat C \coloneq n^{-1} S_\Pa^\ast  S_\Pa$, $\hat C_\lambda\coloneq \hat C + \lambda I_\Hcal$, which equals $n^{-1} A$ for the operator $A$ in the proof of Lemma~\ref{lemRT}. The source condition \eqref{asswellposed} then reads as 
\begin{equation}\label{boundR}
\text{$\| C_\lambda^{-\nu} h_\lambda\|_\Hcal  \le \| C^{-\nu} h_\lambda\|_\Hcal \le R$ and $\|h_\lambda\|_\Hcal\le \|C^\nu\| \|C^{-\nu} h_\lambda\|\le \kappa^{2\nu} R$, }
\end{equation} 
for all $\lambda\ge 0$, as shown in Lemma~\ref{lemregerror0}\ref{lemregerror03}.

We further define $V$ as in \eqref{Vdef}, and $G\coloneq n A$ for $A$ defined in \eqref{Bdefeq}, which we can write as
\[  G ={V} ( {V^\ast} \hat C_\lambda {V} )^{-1} {V^\ast}  .\]
We then have $\hat h_\lambda = \hat C_\lambda^{-1} n^{-1} (S^\ast_\Q\bm 1 - S_\Pa^\ast\bm {p})$ and ${\hat h_{\lambda,\Pi}} = G \hat C_\lambda\hat h_\lambda$. Accordingly, we decompose\footnote{Decomposition \eqref{tilhhG}, and consequently the rest of this proof, is markedly different from the decomposition used in the proof of \cite[Theorem 2]{rud_cam_ros_15}, as we cannot express $S^\ast_\Q\bm 1 - S_\Pa^\ast\bm {p}\neq S_\Pa^\ast J_\Pa h_0$ directly in terms of $h_0$ as they do.}
\begin{equation}\label{tilhhG}
{\hat h_{\lambda,\Pi}} - h_\lambda = G\hat C_\lambda\hat h_\lambda - h_\lambda = (G\hat C_\lambda - I)h_\lambda + G \hat C_\lambda (\hat h_\lambda -h_\lambda).\end{equation}

Using that $G\hat C_\lambda {V} = {V} ( {V^\ast} \hat C_\lambda {V} )^{-1} {V^\ast} \hat C_\lambda {V}= {V}$, we further expand the first term in~\eqref{tilhhG},
\begin{align*}
G\hat C_\lambda - I&= G\hat C_\lambda {V} {V^\ast}  + G\hat C_\lambda (I-{V} {V^\ast})- I \\
&= {V} {V^\ast} - I + G\hat C_\lambda (I-{V} {V^\ast}).
\end{align*}
Taking operator norms, and using that $I-{V} {V^\ast}  $ is an orthogonal projection in $\Hcal$, gives
\begin{align*}
\| (G\hat C_\lambda - I)h_\lambda\|_\Hcal &\le \| ({V} {V^\ast} - I) C_\lambda^\nu \| \| C_\lambda^{-\nu} h_\lambda\|_\Hcal \\
&\quad + \| \hat C_\lambda^{-1/2}\| \|  \hat C_\lambda^{1/2} G \hat C_\lambda^{1/2}\| \| \hat C_\lambda^{1/2} C_\lambda^{-1/2}\| \| C_\lambda^{1/2}(I-{V} {V^\ast})C_\lambda^\nu \| \| C_\lambda^{-\nu} h_\lambda\|_\Hcal \\
&\le  R (1 + \lambda^{-1/2} \theta \Ccal_{\lambda,\Pi}^{1/2} ) \Ccal_{\lambda,\Pi}^\nu ,
\end{align*}
for $\theta\coloneq \| \hat C_\lambda^{1/2} C_\lambda^{-1/2}\|$, and where we used that
\[ \|  \hat C_\lambda^{1/2} G \hat C_\lambda^{1/2}\| \le 1,\]
see \cite[Lemma 8]{rud_cam_ros_15}, and
\[ \| ({V} {V^\ast} - I) C_\lambda^\nu\|\le \Ccal_{\lambda,\Pi}^\nu,\quad \| C_\lambda^{1/2}(I-{V} {V^\ast})C_\lambda^\nu \|\le \Ccal_{\lambda,\Pi}^{1/2+\nu},\]
see bounds on ``B.1'' in the proof of \cite[Theorem 2]{rud_cam_ros_15}, for the computation error $\Ccal_{\lambda,\Pi}= \| C_\lambda^{1/2}(I-{V} {V^\ast})\|^2$ defined in \eqref{defCE}.

Further expanding the second term in \eqref{tilhhG} and
taking norms gives
\begin{align*}
\| G \hat C_\lambda (\hat h_\lambda -h_\lambda)\|_\Hcal&\le  \|\hat C_\lambda^{-1/2}\| \|\hat C_\lambda^{1/2} G \hat C_\lambda^{1/2}\| \| \hat C_\lambda^{1/2} C_\lambda^{-1/2}\| \|C_\lambda^{1/2}(\hat h_\lambda -h_\lambda)\|_\Hcal\le \lambda^{-1/2} \theta \Delta
\end{align*}
for $\Delta \coloneq \|C_\lambda^{1/2}(\hat h_\lambda -h_\lambda)\|_\Hcal$, and we will use that 
\begin{align*}
 \Delta^2 &= \langle \hat h_\lambda -h_\lambda, C_\lambda (\hat h_\lambda -h_\lambda)\rangle_\Hcal  = \| C^{1/2}(\hat h_\lambda -h_\lambda)\|_\Hcal^2 + \lambda \| \hat h_\lambda -h_\lambda\|_\Hcal^2.
\end{align*}
Summarizing we obtain,  
\begin{equation}\label{sum1new}
\| {\hat h_{\lambda,\Pi}} - h_\lambda\|_\Hcal \le  R \big(1 + \lambda^{-1/2} \theta \Ccal_{\lambda,\Pi}^{1/2} \big)\Ccal_{\lambda,\Pi}^\nu   + \lambda^{-1/2}\theta \Delta.
\end{equation}

As seen in Proposition~\ref{propFSGaux}, embedding in $L^2_\Pa$ gives better rates, as\footnote{This follows from the elementary identity $\| J_\Pa h\|_{L^2_\Pa}^2 = \langle h, J_\Pa^\ast J_\Pa h\rangle_\Hcal$ for $h\in\Hcal$.} $\| J_\Pa {\hat h_{\lambda,\Pi}} - J_\Pa h_\lambda\|_{L^2_\Pa}=  \| C^{1/2} ( {\hat h_{\lambda,\Pi}} -   h_\lambda)\|_\Hcal$ and 
\begin{align*}
& \| C^{1/2} (G\hat C_\lambda - I)h_\lambda\|_\Hcal \le \| C^{1/2} C_\lambda^{-1/2}\|  \| C_\lambda^{1/2} ({V} {V^\ast} - I) C_\lambda^\nu \| \| C_\lambda^{-\nu} h_\lambda\|_\Hcal \\
&\quad + \| C^{1/2} C_\lambda^{-1/2}\| \| C_\lambda^{1/2} \hat C_\lambda^{-1/2}\|  \|  \hat C_\lambda^{1/2} G \hat C_\lambda^{1/2}\| \| \hat C_\lambda^{1/2} C_\lambda^{-1/2}\| \| C_\lambda^{1/2}(I-{V} {V^\ast})C_\lambda^\nu \| \| C_\lambda^{-\nu} h_\lambda\|_\Hcal \\
&\le R (1 + \beta \theta  ) \Ccal_{\lambda,\Pi}^{1/2+\nu}  ,
\end{align*}
for $\beta\coloneq \| C_\lambda^{1/2} \hat C_\lambda^{-1/2}\|$ and where we used that $ \| C^{1/2} C_\lambda^{-1/2}\|\le 1$. Similarly, 
\begin{align*}
&\| C^{1/2} G \hat C_\lambda (\hat h_\lambda -h_\lambda)\|_\Hcal\\
&\le  \| C^{1/2} C_\lambda^{-1/2}\| \| C_\lambda^{1/2} \hat C_\lambda^{-1/2}\|  \|\hat C_\lambda^{1/2} G \hat C_\lambda^{1/2}\| \| \hat C_\lambda^{1/2} C_\lambda^{-1/2}\| \|C_\lambda^{1/2}(\hat h_\lambda -h_\lambda)\|_\Hcal\le \beta \theta \Delta.
\end{align*}
Summarizing we obtain,  
\begin{equation}\label{sum2new}
  \| J_\Pa {\hat h_{\lambda,\Pi}} - J_\Pa h_\lambda\|_{L^2_\Pa} \le R (1 + \beta \theta  ) \Ccal_{\lambda,\Pi}^{1/2+\nu}    + \beta \theta \Delta.
\end{equation}

For $\beta$ and $\theta$ we argue as in the proof of \cite[Proposition 2]{rud_cam_ros_15}. Indeed, by \cite[Proposition 7]{rud_cam_ros_15}, it holds $\beta^2\le 1/(1-{ b}(\lambda))$ and $\theta^2\le 1+b(\lambda)$, where $b(\lambda)\coloneq \|C_\lambda^{-1/2}(\hat C - C)C_\lambda^{-1/2}\|$.\footnote{In the proof of \cite[Proposition 2]{rud_cam_ros_15} they wrongly bound $\beta$ by $1/(1-b(\lambda))$, without adjusting for the square. However, as the resulting bound is larger than 1, the  bound on ``$\beta\le 1.5$'' in \cite{rud_cam_ros_15} is correct, but too conservative.} By \cite[Proposition 8]{rud_cam_ros_15}, it holds that $b(\lambda)\le \frac{2(\kappa^2 +\lambda) w}{3\lambda n} + \sqrt{\frac{2\kappa^2 w}{\lambda n}}$ for $w\coloneq \log\frac{8\kappa^2	}{\delta \lambda}$ with sampling probability $1-\delta$, if $\lambda\le \|C\|$.\footnote{We obtain $8\kappa^2$ in the numerator of $w=\log(\cdot)$, as the stated bound in \cite[Proposition 8]{rud_cam_ros_15} holds with probability $1-2\delta$.} Elementary calculations now show that $b(\lambda)\le 1/3$ if $\lambda n \ge  13 \kappa^2 w$.\footnote{We obtain this by consideration of the equivalent quadratic inequality $x^2 - 3\sqrt{c_1} x + 3c_2\ge 0$, for $x=\sqrt{\lambda n}$, $c_1=2\kappa^2 w$, $c_2=2(\kappa^2+\lambda)w/2$. The discriminant is $D=9 c_1-12 c_2= 18\kappa^2 w - 8 (\kappa^2+\lambda)w\ge 10\kappa^2 w$. Hence the larger solution $x_2$ to the quadratic equation ``$=$'' is lower bounded by $x_2\ge (3\sqrt{2\kappa^2 w}+ \sqrt{10\kappa^2 w})/2\ge \sqrt{13\kappa^2 w}$.}
In this case, $\beta\le \sqrt{3/2}<1.23$, $\theta\le \sqrt{4/3} <1.16$, and $\beta\theta\le \sqrt{2}<1.42$.

On the other hand, from Proposition~\ref{propFSGaux}\ref{propFSGaux2},  it follows that 
\[ \Delta \le C_{FS}(\eta,\kappa^2  R)\big( 1 +C_J(\eta,\lambda n) \big)  \lambda^{-1/2} n^{-1/2}\]
with sample probability $1-2\eta$, where we used \eqref{boundR} and that $C_{FS}(\eta,s)$ is increasing in $s$, for the coefficients $C_{FS}$ and $C_J$ given in \eqref{CSdef} and \eqref{CJdef}.

Combining the above with \eqref{sum1new} and \eqref{sum2new} we obtain, with probability $1-\eta$ that the following bounds hold jointly\footnote{We use that $\lambda^{-1/2}\theta\Ccal_{\lambda,\Pi}^{1/2}\le \sqrt{4/3}\sqrt{3} = 2$.} 
\begin{equation}\label{AEboundH}
\begin{aligned}  
\|  {\hat h_{\lambda,\Pi}}- h_\lambda \|_\Hcal &\le     3^{1+\nu}  R \lambda^\nu   + 1.16\, C_{FS}(\eta/6,\kappa^2 R)\big( 1 +C_J(\eta/6,40\kappa^2) \big)  \lambda^{-1} n^{-1/2}
  \end{aligned}
\end{equation}
and 
\begin{equation}\label{AEboundJ}
\begin{aligned}
  \| J_\Pa {\hat h_{\lambda,\Pi}} - J_\Pa h_\lambda\|_{L^2_\Pa} & \le(1 + \sqrt{2} ) 3^{1/2+\nu}   R \lambda^{1/2+\nu} \\
  &\quad + \sqrt{2}  C_{FS}(\eta/6,\kappa^2 R)\big( 1 +C_J(\eta/6,40\kappa^2) \big)  \lambda^{-1/2} n^{-1/2}
  \end{aligned}
\end{equation}
if 
\begin{equation}\label{lambdancond}
\lambda n \ge 13 \kappa^2 \log\frac{24\kappa^2	}{\eta \lambda}
\end{equation}
and $m\ge C_{CE}(\eta/3,\lambda,n)$. We also used that, by assumption $\lambda\le \|C\|\le \kappa^2$, we have $\lambda n \ge 13 \kappa^2 \log\frac{24 	}{\eta }> 40 \kappa^2$ and $C_J(\eta,s)$ is decreasing in $s$. In fact, $C_J(\eta/6,40\kappa^2)\le  (1+1/40) (\eta/6)^{-1/2}< 0.07 \eta^{-1/2}$.\\
Setting $\lambda = \| C\| n^{-1/(2+2\nu)}$, so that $\lambda n = \|C\| n^{(1+2\nu)/(2+2\nu)}$, shows that condition~\eqref{lambdancond} boils down to \eqref{lambdancondn}, and $m\ge C_{CE}(\eta/3,\lambda,n)$ reads as in \eqref{eqmLB}, see \eqref{rateCCE}. Moreover, the right hand sides in \eqref{AEboundH} and \eqref{AEboundJ} are bounded by $C_{AE,\Hcal}(\eta,\nu)n^{-\nu/(2+2\nu)}$ and $C_{AE,L^2_\Pa}(\eta,\nu)n^{-(1+ 2\nu)/(4+4\nu)}$ for the coefficients given in \eqref{defCAEH} and \eqref{defCAEL}, respectively.\footnote{We use the numerical bound of $(1+\sqrt{2}) \sqrt{3} <4.19$.} We then decompose $\|  {\hat h_{\lambda,\Pi}}- h_0 \|_\Hcal\le \|  {\hat h_{\lambda,\Pi}}- h_\lambda \|_\Hcal+\|    h_\lambda- h_0\|_\Hcal$ and combine with \eqref{hb3s} to obtain \eqref{eqHrates}. For the total error rates in \eqref{totalrates} we note that by orthogonality \eqref{SQAE} we have $\Ecal({\hat h_{\lambda,\Pi}}) -  \Ecal(h_0) = \| J_\Pa \hat h_{\lambda,\Pi} -J_\Pa h_0\|_{L^2_\Pa}^2$. We then decompose $\| J_\Pa \hat h_{\lambda,\Pi} -J_\Pa h_0\|_{L^2_\Pa}\le \| J_\Pa \hat h_{\lambda,\Pi} -J_\Pa h_\lambda\|_{L^2_\Pa}+\| J_\Pa h_\lambda-J_\Pa h_0\|_{L^2_\Pa}$ and combine with \eqref{hb4s} to obtain \eqref{totalrates}. This completes the proof of Theorem \ref{thmAEbound}.

\subsection{Proof of Theorem \ref{lthmnullh}} 
\label{app:proof_thm9}
That \eqref{nullh} implies \eqref{nullh1} follows from \eqref{hlambdaeqpre}. Moreover, it trivially also implies well-posedness \eqref{asswellposed} for $h_0=0$, and thus $R=0$ in \eqref{boundR}.

\ref{lthmnullh1}: This now follows from \eqref{AEboundH}, \eqref{AEboundJ}, and \eqref{lambdancond}, and because \eqref{hlambdaeqsampleLRcoor1} and \eqref{VastVeq} imply that the norms $\|{\hat h_{\lambda,\Pi}} \|_\Hcal$ and $\|J_\Pa{\hat h_{\lambda,\Pi}} \|_{L^2_\Pa}$ are given by the left hand sides of \eqref{eqtestnullnew} and \eqref{eqtestnullJP}, respectively. Here we approximate the $L^2_\Pa$-norm in \eqref{AEboundJ} asymptotically for large $n$ by the empirical norm $n^{-1/2}\|\cdot \|_2$ induced by the sample $\bm z_\Pa$ and by replacing $J_\Pa$ by $S_\Pa$. Note that $n^{-1/2}$ cancels on both sides of \eqref{AEboundJ}.

\ref{lthmnullh2}: Hypothesis \eqref{nullh} implies $J_\Q^\ast 1=J_\Pa^\ast {p}$, by Lemma~\ref{lemJast0Jast}. Hence Lemma~\ref{lemAC}~\ref{lemAC2} implies that, asymptotically for large $n$,
\begin{equation}\label{ulamhyp}
    u_\lambda = n^{-1/2}\big(S_\Q^\ast \bm 1 - S_\Pa^\ast \bm {p} \big)\sim\Ncal(0,Q_\lambda)
\end{equation}     
is normally distributed, where the covariance operator \eqref{Clamdef} simplifies to 
\[ Q_\lambda = J_\Q^\ast J_\Q - (J_\Q^\ast 1)\otimes (J_\Q^\ast 1)  + J_\Pa^\ast \diag({p})^2 J_\Pa - (J_\Pa^\ast {p})\otimes (J_\Pa^\ast {p}).\]  
We estimate $Q_\lambda$ by its sample analogue given by
\[ \hat Q_\lambda = n^{-1} S_\Q^\ast S_\Q - (n^{-1}S_\Q^\ast \bm 1)\otimes (n^{-1} S_\Q^\ast \bm 1)  + n^{-1}S_\Pa^\ast \diag(\bm {p})^2 S_\Pa - (n^{-1}S_\Pa^\ast \bm {p})\otimes (n^{-1} S_\Pa^\ast \bm {p}).\] 
To reduce the dimension, we project the sample variable $u_\lambda\mapsto   P_\Pi u_\lambda \eqqcolon V  v_\lambda$ on the subspace $\Hcal_\Pi$, with coordinate vector $v_\lambda= V^\ast u_\lambda$ in $\R^\ell$, see \eqref{R1} and \eqref{Vdef}. Using the identities in the proof of Lemma~\ref{lemapplocal}, we obtain that $v_\lambda$ is given by \eqref{vlamN} and normally distributed with mean zero and $\ell\times \ell$-covariance matrix \eqref{SigN}.

\ref{lthmnullh3}: This follows directly from \ref{lthmnullh2}.

\ref{lthmnullh4}: From \ref{lthmnullh2} and \ref{lthmnullh3}, we have that $\bm A v_{\lambda}\sim N(\bm 0, \bm W)$, from which we obtain the mean
$\mathbb E [S_\Gamma ]=\trace \bm W$ and variance $\mathbb V[S_\Gamma]= 2 \trace (\bm W^2)$. Matching moments of the approximating Gamma distribution $\Gamma(\alpha,\beta)$ gives shape $\alpha=\mathbb E [S_\Gamma ]^2/\mathbb V [S_\Gamma]$ and scale $\beta=\mathbb V [S_\Gamma]/\mathbb E [S_\Gamma ]$, as claimed.

\subsection{Proof of Lemma \ref{lemassu1XY}}
\label{app:proof_lemma11}
Denote by $f_X(x)=\int_\Ycal f(x,y) \mu_\Ycal( d  y)$ and $f_Y(x)=\int_\Xcal f(x,y) \mu_\Xcal( d  x)$ the marginal densities, such that $\Pa_X( d  x)=f_X(x)\mu_\Xcal( d  x)$ and $\Pa_Y( d  y)=f_Y(y)\mu_\Ycal( d  y)$. Define their positivity sets $S_X\coloneqq\{ f_X>0\}$ and $S_Y\coloneqq \{f_Y>0\}$. We claim that $\Pa_{(X,Y)}[S_X\times S_Y]=1$. Indeed, let $B$ be any measurable set contained in the complement $\Zcal\setminus (S_X\times S_Y)$. Then 
\begin{align*}
\Pa_{(X,Y)}[B] &\le \Pa_{(X,Y)}[B\cap (S_X^c \times \Ycal)] + \Pa_{(X,Y)}[B\cap (\Xcal \times S_Y^c)] \\
& \le \Pa_{(X,Y)}[ S_X^c \times \Ycal] + \Pa_{(X,Y)}[\Xcal \times S_Y^c] =\Pa_X[S_X^c]+\Pa_Y[S_Y^c]=0,
\end{align*}
which proves the claim. Hence we can replace $f$ in \eqref{eqlemass} by $f 1_{S_X\times S_Y}$, where $1_B$ denotes the indicator function of a set $B$. By definition we have $1_{S_X\times S_Y}(x,y)=0$ for any $(x,y)\in\Zcal$ such that $f_X(x) f_Y(y)=0$. As on the other hand we have $(\Pa_X\otimes \Pa_Y)( d  x, d  y) = f_X(x) f_Y(y)(\mu_X\otimes \mu_Y)( d  x, d  y)$, this proves \eqref{assu1XY}.

\section{Independence test distributions}\label{sec:independencedistributions} 
This appendix contains the distributions for the independence tests in Section~\ref{sec:emp-two-sample}.
\begin{itemize}
 \item \textbf{Independent clouds}: \[X = X_0 + \varepsilon_X, Y = Y_0 + \varepsilon_Y, \]where  $X_0$ and  $Y_0$  take values $\{-1, 1\}$ with probability 1/2, and  $\varepsilon_X$  and  $\varepsilon_Y$ are i.i.d.\ standard normal. 
 \item \textbf{W}: \[ X \sim \text{Unif}(-1,1), Y = C(X^2 - 0.5)^2 + \varepsilon, \qquad \varepsilon \sim  \Ucal (0,1). \] 
 \item \textbf{Diamond}: \[U_1 = U \cos \frac{\pi}{4} + V \sin \frac{\pi}{4}, 
V_1 = -U \cos \frac{\pi}{4} + V \sin \frac{\pi}{4},\]
 where $U$, $V \sim \Ucal(-1, 1)$ are independent uniform random variables. Let $(X, Y) = (U_1, V_1)$ if $\varepsilon < C$, and $(X, Y) = (U_2, V_2)$ otherwise, where $U_2$, $V_2 \sim \Ucal(-1, 1)$ and $\varepsilon \sim \Ucal(0, 1)$ are all i.i.d.\ random variables.
 \item \textbf{Parabola}: \[X \sim \Ucal(-1,1), 
Y = CX^2 + \varepsilon,\qquad \varepsilon \sim \Ucal(0,1).\]
\item \textbf{Two parabola}: \[X\sim \Ucal(-1, 1), \, 
 Y = (CX^2+\varepsilon)V,\]
 where $V$ takes values $\{-1, 1\}$ with probabilities $1/2$, and  $\varepsilon \sim \Ucal(0,1)$.
 \item \textbf{Circle}: \[X = C \sin(2\pi U) + \varepsilon_1, \, 
Y = 4.2 \cos(2\pi U) + \varepsilon_2,\]
where $U \sim \Ucal(-1, 1)$, $\varepsilon_1$ and $\varepsilon_2$ are i.i.d.\ standard normal.
\item \textbf{Variance}: \[Y = \varepsilon \sqrt{CX^2 + 1},\]
where $X$ and $\varepsilon$ are i.i.d.\ standard standard normal.
\item \textbf{Log}: \[Y = C \log X^2  + \varepsilon, \]
where $X$ and $\varepsilon$ are i.i.d.\ standard standard normal.
\end{itemize}

\end{document}